\documentclass{article}

\usepackage[preprint]{neurips_2026}

\usepackage[utf8]{inputenc} 
\usepackage[T1]{fontenc}    
\usepackage[pagebackref]{hyperref} 
\usepackage{url}
\usepackage{booktabs}       
\usepackage{amsfonts}       
\usepackage{nicefrac}       
\usepackage{microtype}      
\usepackage{xcolor}         
\usepackage{graphicx}
\usepackage{tcolorbox}
\usepackage{capt-of} 
\usepackage{subcaption}
\usepackage{enumitem}
\usepackage{amsmath}
\usepackage{amssymb}
\usepackage{mathtools}
\usepackage{amsthm}

\hypersetup{
    colorlinks,
    linkcolor=cyan,
    citecolor=cyan,
    urlcolor=cyan,
    filecolor=cyan,
}

\usepackage{neutral}
\usepackage{wrapfig}
\definecolor{darkgreen}{rgb}{0.0, 0.5, 0.0}
\definecolor{darkred}{rgb}{0.5, 0.0, 0.0}
\definecolor{darkgrey}{rgb}{0.5, 0.5, 0.5}
\usepackage{colortbl}
\definecolor{forestgreen}{RGB}{34,139,34}

\newcommand{\dgs}[1]{{\scriptsize{\textcolor{darkgreen}{(#1)}}}}

\newcommand{\rgs}[1]{{\scriptsize{\textcolor{darkred}{(#1)}}}}

\newcommand{\ggs}[1]{{\scriptsize{\textcolor{darkgrey}{(#1)}}}}

\usepackage{booktabs}
\usepackage{multirow}

\tcbset{boxsep=0mm,boxrule=0pt,colframe=white,arc=0mm,left=0.5mm,right=0.5mm}
\DeclareRobustCommand{\mybox}[2][gray!10]{%
	\begin{tcolorbox}[   %
		left=0pt,
		right=0pt,
		top=0pt,
		bottom=0pt,
		colback=#1,
		colframe=#1,
		enlarge left by=0mm,
		boxsep=5pt,
		arc=0pt,outer arc=0pt,
		]
		#2
	\end{tcolorbox}
}

\usepackage{url}
\usepackage{amsmath}
\usepackage{mathtools}
\usepackage{amsthm}
\usepackage{thmtools}
\usepackage{thm-restate}
\usepackage{cancel}
\usepackage{amsthm}
\usepackage{amssymb}
\newtheorem{proposition}{Proposition}

\newtheorem{lemma}{Lemma}

\newtheoremstyle{theoremdd}
  {\topsep}
  {\topsep}
  {\itshape}
  {0pt}
  {\bfseries}
  {. }
  { }
  {\thmname{#1}\thmnumber{ #2}\textnormal{\thmnote{ (#3)}}}
\theoremstyle{theoremdd}

\title{Towards Understanding Self-Pretraining\\ for Sequence Classification}

\author{%
Omar Coser\thanks{Unit of Artificial Intelligence \& Computer Systems, Università Campus Bio-Medico di Roma, Via Álvaro del Portillo, 21, Rome, 00128, Italy. Unit of Advanced Robotics and Human-Centered Technologies, Università Campus Bio-Medico di Roma, Via Álvaro del Portillo, 21, Rome, 00128, Italy. Corresponding author: \texttt{omarcoser10@gmail.com}.}
\And
Loredana Zollo\thanks{Unit of Advanced Robotics and Human-Centered Technologies, Università Campus Bio-Medico di Roma, Via Álvaro del Portillo, 21, Rome, 00128, Italy.}
\And
Paolo Soda\thanks{Unit of Artificial Intelligence \& Computer Systems, Università Campus Bio-Medico di Roma, Via Álvaro del Portillo, 21, Rome, 00128, Italy. Department of Diagnostics and Intervention, Radiation Physics, Biomedical Engineering, Umeå University, Universitetstorget, 4, Umeå, 40196, Sweden.}
\And
Antonio Orvieto\thanks{Max Planck Institute for Intelligent Systems. ELLIS Institute Tübingen. Tübingen AI Center.}
}

\begin{document}

\maketitle

\begin{abstract}
Amos et al.~(2024) showed that the accuracy of Transformer models in sequence classification can be significantly improved by first pretraining with a masked token prediction objective \textit{without external data or augmentation}, a procedure referred to as self-pretraining (SPT). While the primary objective of Amos et al. (2024) was to showcase that Transformers can achieve strong performance on the Long-Range Arena (LRA), their pipeline raises more fundamental questions: How does SPT drive optimization to better solutions? Why can standard supervised training fail in Transformers? To better understand this, we replicate and systematically ablate the findings of Amos et al.~(2024). Our ablations suggest that a central bottleneck in the studied settings is not depth or generalization alone, but the ability of label supervision to learn useful query-key Attention patterns from random initialization. With a minimal setup, we identify learning proximity interactions -- turning absolute positional encodings into proximity-biased Attention scores -- as a key source of the improvements brought by SPT. Finally, in a simplified theoretical setup, we show that label supervision can be locally blind to certain Attention-score directions that are instead detectable through masked reconstruction.
\end{abstract}

\vspace{-3mm}
\section{Introduction}
\vspace{-1mm}
The Long-range Arena~\citep{tay2020efficient} has played an important role in the development of new efficient token-mixing strategies over the last few years. A prime example is S4~\citep{gu2022efficiently}, the first state-space models (SSMs), which drew attention after surpassing Transformers~\citep{vaswani2017attention} on the LRA, achieving a $20\%$ average increase in accuracy. S4 and its later variants, many of which were developed and evaluated using the LRA benchmark~\citep{smith2023simplified, gu2022parameterization, poli2023hyena, orvieto2023resurrecting}, laid the groundwork for modern architectures such as Mamba~\citep{gu2024mamba,dao2024transformers}, GLA~\citep{yang2024gated}, RWKV~\citep{peng2024eagle,peng2025rwkv} and (Gated) DeltaNet~\citep{yang2024parallelizing, yang2024gated} among others. While the latest developments are mainly inspired by language modeling performance~\citep{waleffe2024empirical, poli2024mechanistic}, researchers have maintained interest in explaining and studying the gap between Transformers and SSMs on LRA, see e.g.~\citet{zimerman2023long,alonso2025state, qin2024hgrn2}.

In their seminal contribution, awarded the \textit{Outstanding Paper Award at ICLR 2024}, \citet{amos2023never} showed that training Transformer models from scratch~(as done in the LRA benchmark) leads to an under-estimation of their performance and demonstrates that dramatic gains can be achieved with a pretraining $\to$ finetuning setup~(see Tb.~\ref{tb:reproduce}). At first glance, this finding may not seem too surprising -- yet, \textbf{pretraining is performed here on the same data}~(self-pretraining, denoted as \textbf{SPT}), that is: for a fixed task, the Transformer is first asked to learn to predict masked tokens (or the next token) in the sequence, and only later is trained on task labels. This procedure differs substantially from the usual pretraining paradigm~(a large, heterogeneous pretraining corpus and a small, specialized fine-tuning dataset, e.g., \citet{brown2020language}). 

However, the idea of SPT~(Fig.~\ref{fig:SPT}) is not attributable to~\citet{amos2023never}: it can already be found in~\citet{el2021large}, who showed the efficacy of a self-pretraining pipeline in vision. Self-pretraining as a new paradigm then appeared in~\citet{peng2023improving}, showing that SPT models achieve competitive performance on downstream speaker verification tasks with only one-third of the data compared to standard pretraining. Concurrently, \citet{zhou2023self} demonstrated that self-pretraining improves performance across diverse medical image tasks. Further, \citet{krishna2023downstream} demonstrated that pretraining language models directly on downstream datasets could match performance after pretraining on massive external data, highlighting that much of the observed \textit{performance gains may be driven by the pretraining objective itself rather than the data volume and its heterogeneity}.

\begin{figure}[t]
    \vspace{-3mm}
    \centering
\includegraphics[width=0.85\linewidth]{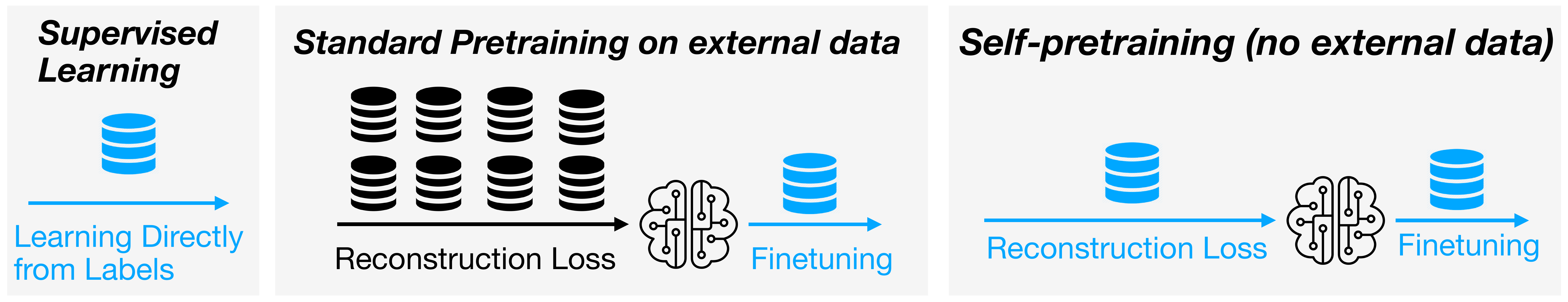}
    \caption{Summary of training pipelines for classification. SPT is the method we study in this work.}
    \label{fig:SPT}
    \vspace{-0mm}
\end{figure}

\begin{table*}[h]
\vspace{-2mm}
\caption{\textbf{Replication of the results} of~\citet{amos2023never}. We use their codebase and similar training settings, see \S~\ref{sec:pre}. We find good agreement with their results and hyperparameters. The numbers reported always refer to final peak held-out validation accuracy, with the values in {\color{darkgreen}green} parentheses referring to the absolute performance increase with respect to the corresponding from-scratch result.}
\label{tb:reproduce}
\vspace{-1mm}
\begin{center}
\begin{small}
\begin{sc}
\resizebox{0.95\textwidth}{!}{%
\begin{tabular}{llllll}
\toprule
\textbf{Training Modality} & ListOps & CIFAR10& PathFinder &Retrieval &Text \\
\midrule
 {\scriptsize Standard (from scratch)}& 0.402 &0.499 &0.67 & 0.776 & 0.653\\
 {\scriptsize +SPT~(Amos et al.) } &0.59 \dgs{+0.19} &0.74 \dgs{+0.24} & 0.88 \dgs{+0.21} &0.88 \dgs{+0.11}&0.89 \dgs{+0.24} \\
{\scriptsize +SPT~(Our reproduction) } & 0.56 \dgs{+0.16}& 0.72 \dgs{+0.22}& 0.85 \dgs{+0.18} &0.88 \dgs{+0.11} &0.89 \dgs{+0.24}\\
\bottomrule
\end{tabular}
}
\end{sc}
\end{small}
\end{center}
\vspace{-2mm}
\end{table*}

Compared to the applied research cited above, the objective of~\citet{amos2023never} was not to show how to reach state-of-the-art performance or to eliminate the need for pretraining on a massive external dataset. Instead, they show that on challenging tasks such as the LRA, \textbf{Transformers, in contrast to other models such as SSMs, \textit{critically} benefit from data-informed initialization in supervised learning}. This is a valuable practical insight; it \textbf{hints at a potential difficulty in learning Attention masks from labels alone}, an interesting phenomenon that warrants further investigation.


To shed light on the issue, we first pose the following questions that isolate depth, data source, and optimization effects. We mostly explore these on the LRA, as it is challenging for Attention-based models~\citep{tay2020long,zimerman2023long,amos2023never}.

\vspace{-1mm}
\begin{enumerate}[leftmargin=2em,itemsep=0.05em]
    \item \textit{Role of depth:} Is the boost in validation accuracy observed when including SPT only observable in deep models and after a self-pretraining budget of hundreds of epochs? Can we observe and try to understand this effect in shallow models, perhaps with only one layer?
    \item \textit{Role of data:} Is the self-pretraining boost due to learning data-specific features, or can it also be observed when pretraining e.g. on another~(small-scale) dataset in the same benchmark?
    \item \textit{Role of optimization:} Is poor from-scratch validation performance an optimization issue~(e.g., reconstruction loss allows gradients to flow better), or is it instead the case that both from-scratch training and self-pretraining fit the data, yet SPT leads to better generalization? 
\end{enumerate}
\vspace{-1mm}
Towards answering these questions, we make the following contributions:
\begin{itemize}[leftmargin=2em,itemsep=0.05em]
    \item We replicate the results of~\citet{amos2023never} in Tb.~\ref{tb:reproduce}, and show additional evidence for SPT beyond the LRA in \S~\ref{sec:other_datasets}. We then present ablations regarding pretraining duration~(\S~\ref{sec:duration}), model depth and data source~(\S~\ref{sec:depth}), weight freezing~(\S~\ref{sec:freeze}) and mixed initializations~(\S~\ref{sec:subset}).
    \item Throughout \S~\ref{sec:ablations} we summarize our main findings in ``Takeaway boxes''. Our results show that SPT gains persist across several LRA settings~(various model depths, different data sources and pretraining durations). By re-initializing subsets of model weights to a uniform distribution after self-pretraining and before finetuning, our ablations identify Attention parameters, and especially $W_Q, W_K$ in one-layer models, as major carriers of the SPT gain.
    \item We further investigate the Attention patterns unlocked by SPT. To do that, we construct in \S~\ref{sec:toy} a minimal $1$-layer synthetic example. A careful analysis of this setup reveals that self-pretraining can learn Attention weights $W_Q, W_K$ that convert additive sinusoidal positional encodings into a relative, proximity-biased Attention pattern that finetuning can build on. 
    \item Why can learning specific Attention patterns fail with pure label supervision? In \S~\ref{sec:theory} we provide a theoretical abstraction and analysis in a simplified yet insightful setup.
\end{itemize}
\vspace{-2mm}
We believe our investigation can spark further theoretical discussions of the statistical and optimization properties of this promising new paradigm.

\section{Preliminaries}
\label{sec:pre}

\vspace{-1mm}
\paragraph{Model.} Let $B$ denote the batch size, $L$ the sequence length and $D$ the embedding dimension. Given embedded inputs $X\in\mathbb{R}^{B\times L\times D}$,~\citet{amos2023never} inject absolute positional information once as $Z^{(0)}=X+P$.
The backbone is a pre-norm~\citep{xiong2020layer} residual stack of $T$ blocks, each consisting of an Attention layer followed by a position-wise MLP.
For each block,
$Z' = Z + \mathrm{Drop}(\mathrm{MHA}(\mathrm{LN}(Z)))$ and
$Z^{+} = Z' + \mathrm{Drop}(\mathrm{MLP}(\mathrm{LN}(Z')))$.
Attention is multi-head: for $H$ heads, $d=D/H$, $Q_h=Z(W_Q^{h})^\top,\ K_h=Z(W_K^{h})^\top,\ V_h=Z(W_V^{h})^\top$, the operation is $$\mathrm{MHA}(Z)=\mathrm{Concat}_{h}
\left(
\mathrm{softmax} \Big(\tfrac{Q_hK_h^\top}{\sqrt{d}}\Big)V_h
\right)W_O.$$
For classification, the sequence $Z^{(T)}$ is mean-pooled, and the output vector is used for classification. In self-pretraining, the architecture is exactly the same, yet a fixed fraction $\mathcal{M}$ of the input sequence is masked and the model is trained to reconstruct its original continuous values. $Z^{(T)}$ is projected token-wise and evaluated using $\mathcal{L}_{\mathrm{SPT}}
=
\sum_{b=1}^B \sum_{t\in\mathcal{M}}
\ell\!\left(W_{\mathrm{spt}} Z^{(T)}_{b,t,:},\, X_{b,t,:}\right)$.

\vspace{-3mm}
\paragraph{Training setup.} The LRA consists of 6 classification tasks, where if inputs are not strictly of a sequential nature~(e.g. Image, the black/white version of CIFAR10) they are first flattened to produce a sequence. For a full description of these tasks, please refer to our Appendices~\ref{sec:tasks} and~\ref{sec:exp_details}. We do not consider the PathX dataset, where the full SPT+finetuning pipeline is roughly $16\times$ more expensive than PathFinder. Our LRA conclusions thus apply to the other five tasks, not the full six-task benchmark. Our first step towards a closer inspection of the results by~\citet{amos2023never} is a replication of their findings, for which we use their codebase and the same Transformer settings~(width, depth, heads, etc). We follow~\citet{amos2023never} and consider self-pretraining through a masked\footnote{Following~\citet{amos2023never}, $50\%$ for images, $15\%$ for text tasks, and $10\%$ in ListOps.} sequence modeling objective. After a 200 epochs SPT stage where the model learns to reconstruct dataset-specific inputs~(self-pretraining), we perform finetuning for 100 epochs and compare the results with those of from-scratch training for 100 epochs~(no SPT). We use AdamW~\citep{loshchilov2017fixing} and found the hyperparameters provided by~\citet{amos2023never} to be well-tuned both in the self-pretraining and finetuning stages, and their results align well with our reproduction in Tb.~\ref{tb:reproduce}. 

Following prior LRA work, we use the public held-out split as the validation/evaluation split. Following~\citet{amos2023never}, scalar LRA results report the \textit{peak held-out accuracy} attained within the fixed finetuning budget. For the CIFAR 10 and PathFinder ablations that support our Attention-specific claims, namely the freezing, hybrid-initialization, and one-layer $QK$ experiments, we report mean $\pm$ standard deviation over $5$ random seeds. Results without a $\pm$ value, including the larger LRA sweeps and the ListOps, Retrieval, and Text entries, are single-seed runs due to computational cost and low observed variation in the maximum accuracy over iterations on CIFAR 10 and PathFinder. 

To motivate our study beyond the LRA, we provide evidence of the SPT gap on \textit{additional datasets} in~\S\ref{sec:other_datasets}, covering human activity and physiological signal classification~\citep{dempster2019rocket}.


\section{Ablations}
\label{sec:ablations}
We present here four sets of ablations on the LRA, bringing insights into SPT mechanisms: varying pretraining duration~(\S~\ref{sec:duration}), changing model depth and data source~(\S~\ref{sec:depth}), freezing layers when training from scratch~(\S~\ref{sec:freeze}), and using only a subset of self-pretrained weights~(\S~\ref{sec:subset}).

\begin{figure}[t]
\vspace{-0mm}
    \centering
    \begin{minipage}[t]{0.33\linewidth}
        \centering
        \includegraphics[width=\linewidth]{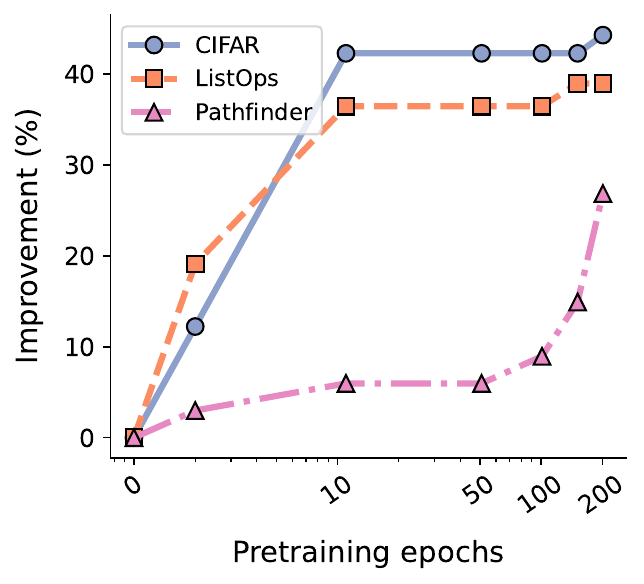}
        \caption{\textbf{SPT duration} ablation. We vary self-pretraining epochs before 100 epochs of finetuning. CIFAR10 and ListOps benefit after only a few epochs, PathFinder needs longer pretraining.}
\label{fig:pt_epochs}
        \label{fig:spt_epochs}
    \end{minipage}
    \hfill
    \begin{minipage}[t]{0.6\linewidth}
        \centering
        \includegraphics[width=\linewidth]{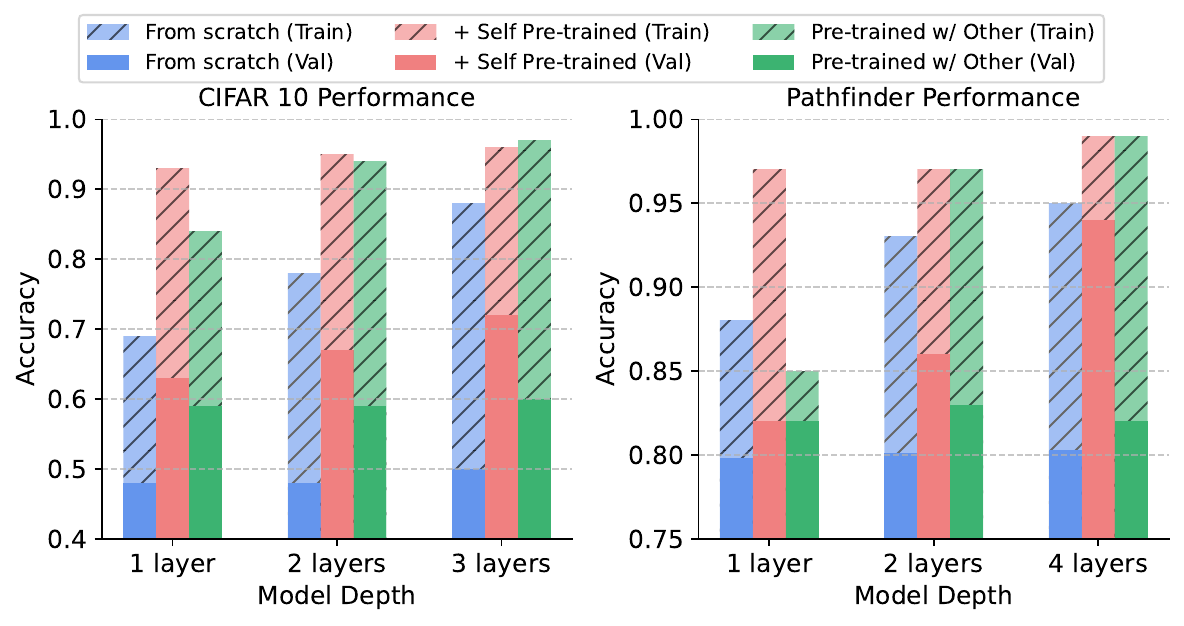}
        \caption{\textbf{Depth and dataset-swap} ablation. We compare training from scratch, same-dataset SPT, and cross-dataset pretraining before finetuning. SPT can improve both train and validation accuracy even for 1-layer models and with swapped pretraining data, pointing to an optimization benefit beyond task-specific representation learning.}
        \label{fig:depth}
    \end{minipage}
\vspace{-1mm}
\end{figure}

\subsection{Self-pretraining Duration}
\label{sec:duration}
\vspace{-2mm}

In our experiments leading to Tb.~\ref{tb:reproduce}, we noticed that pretraining can be computationally expensive\footnote{In CIFAR the wall-clock time for 1 epoch of pretraining on our hardware~(a single A100 with $80$GB of RAM) is 48s, for ListOps is 19.1m and for PathFinder is 3.20min. During finetuning, the time for 1 epoch in CIFAR is 56s, for ListOps is 18.95m and for PathFinder is 3min. Since as~\citet{amos2023never} we first pretrain for 200 epochs and then finetune on 100, the total time compared to basic training from scratch for 100 epochs is tripled.}. Our results in Fig.~\ref{fig:pt_epochs} show that \textit{only 10 epochs} of self-pretraining are needed to significantly boost performance on CIFAR10 and ListOps. Instead, for PathFinder, major gains can only be observed after 100 epochs. Improvements can already be observed after a single epoch. Our hyperparameters are kept the same as in Tb.~\ref{tb:reproduce}, except in the $1$-epoch setting where we reduced the learning rate.

\vspace{-1mm}
\mybox[gray!8]{
\textbf{Takeaway 1:} SPT can yield substantial improvements even after a single epoch of pretraining.}
\subsection{Evidence for an optimization bottleneck (Depth + Dataset Swap)}
\label{sec:depth}

A natural explanation for the observed increase in validation accuracy could be the ability of self-pretraining to learn crucial patterns and \textbf{hierarchical representations} of the data, providing a \textit{better data-distribution-informed initialization} when learning labels~\citep{hubel1962receptive,hinton2006reducing}. While in \S~\ref{sec:duration}, we showed how such representations may arise with just a few epochs of self-pretraining, here we inspect the efficacy of self-pretraining as we vary the \textbf{number of layers} in the model --- testing whether SPT success is contingent on the presence of a hierarchy in representations. At the same time, to test whether self-pretraining is indeed learning task-specific representations, we consider \textbf{swapping pretraining datasets}.\\
For each target dataset $A$, we compare: training from scratch on $A$, SPT on $A$ followed by finetuning on $A$, and pretraining on another dataset $B\ne A$ followed by finetuning on $A$. Results are in Fig.~\ref{fig:depth}. 
 \vspace{-2mm}
 \begin{enumerate}[leftmargin=2em,itemsep=0.1em]
     \item We observe \textbf{gains even when pretraining on a single layer}. While validation accuracy after fine-tuning improves by only $\sim2\%$ after self-pretraining a 1-layer model on PathFinder, on CIFAR10 we observe a $\sim15\%$ increase. Even though the gap in deeper models is larger, especially for PathFinder, we find it interesting that it can be observed at such small scales.
     \item  On both datasets, \textbf{from-scratch validation accuracies are almost independent of the number of layers}. Moreover, there is a substantial gap in training accuracy between from-scratch and SPT performance, pointing to optimization issues.
     \item Some gains in validation accuracy after finetuning can be \textbf{observed when switching data source}. This suggests that SPT may improve trainability by moving the model away from a suboptimal initialization, rather than only by learning task-specific representations.
 \end{enumerate}

 \vspace{-1mm}

\mybox[gray!8]{
\textbf{Takeaway 2:} The SPT boost is not necessarily related to generalization: SPT can greatly improve \textit{training} accuracy after finetuning -- even when models are just $1$-layer deep. This can also hold when pretraining on a  different data distribution from another LRA task. This hints at optimization problems when training Attention-based sequence models directly from labels.}
 
\subsection{Freezing layers: from-scratch sensitivity analysis}
\label{sec:freeze}
In \S~\ref{sec:depth}, we saw that short pretraining, either on task data or on data with a different distribution, can improve training accuracy~(and sometimes, validation accuracy) when learning from labels. \citet{amos2023never} observed that models with different sequence mixers, such as SSMs, can be less sensitive to SPT. This suggests that the trainability issue may be tied to softmax Attention itself. To this end, we consider a drastic experiment: training from scratch after freezing all Attention weights in the architecture. If attention is not effectively learned from labels, then freezing Attention at random initialization should have little effect on performance. This finding would also be consistent with insights from early-stage vision Transformer training~\citep{touvron2021training}, indicating that Transformers are data-hungry and may benefit from additional supervision beyond hard labels. 

In Tb.~\ref{tab:freezing}, training the Attention layers has little effect when training from scratch, except on Pathfinder. In some cases, such as CIFAR-10, freezing Attention even improves performance. By contrast, also freezing the intermediate MLP layers usually decreases performance, as expected.

\begin{table*}[h]
\vspace{-2mm}
\caption{Comparison of different \textbf{freezing configurations} when training deep models \textbf{from scratch}. Reported are val. accuracy results, in the setting described in \S~\ref{sec:pre}. We either clamp only Attention weights or Attention weights and intermediate feedforward layers to random initialization, and learn other parameters from labels. Encoder, normalization weights, and decoder layers are always trained. Entries with $\pm$ report mean $\pm$ standard deviation across 5 seeds; entries without $\pm$ are single-seed.}
\label{tab:freezing}
\begin{center}
\begin{small}
\resizebox{0.9\textwidth}{!}{%
\begin{sc}
\begin{tabular}{l l l l l l}
\toprule
\textbf{Dataset} & \textbf{From Scratch}  & \textbf{with Att. frozen} & \textbf{with Att. \& MLP frozen}\\
\midrule
CIFAR10      & $0.503 \pm 0.002$  & $0.535 \pm 0.004$ \dgs{+0.032}  & $0.468 \pm 0.008$ \rgs{-0.035} \\
PathFinder    & $0.672 \pm 0.003$  & $0.50\pm 0.00$ \rgs{-0.172}& $0.50\pm 0.00$ \rgs{-0.172}\\
ListOps       & 0.40  & 0.40 \ggs{+0.00}& 0.30 \rgs{-0.10}\\
Retrieval     & 0.78  & 0.78 \ggs{+0.00}& 0.67 \rgs{-0.09}\\
Text          & 0.65  & 0.65 \ggs{+0.00}& 0.66 \dgs{+0.01}\\
\bottomrule
\end{tabular}
\end{sc}
}
\end{small}
\end{center}
\vspace{-3mm}
\end{table*}

\mybox[gray!8]{
\textbf{Takeaway 3:} With the exception of PathFinder, from-scratch performance is unaffected by freezing Attention weights to random. This indicates that supervised training makes limited effective use of the Attention parameters in these settings and struggles to combine sequential information.}

\subsection{Using SPT initialization only on a subset of weights}
\label{sec:subset}
\S~\ref{sec:freeze} suggests that the gains from self-pretraining are not uniformly distributed across model components, with the Attention mechanism emerging as particularly important. Motivated by this observation, we study hybrid initializations in which, for each layer, only selected subsets of parameters are initialized from SPT weights, while all remaining parameters are randomly initialized.

The results in Tb.~\ref{tab:subsets} show a clear separation between architectural components. \textbf{Initializing only MLP-related parameters yields little to no improvement over random initialization} (see \textit{MLP+Norm+Encoder}), indicating that self-pretraining provides limited transferable signal to the MLP branch. \textbf{In contrast, initializing the Attention parameters has a pronounced effect} (see \textit{Attention+Norm+Encoder}). In particular, using self-pretrained values solely for the Attention weights $W_K, W_Q, W_V,$ and $W_O$ can already lead to substantial accuracy gains in CIFAR10. While performance close to full self-pretraining is achieved only when Normalization and embedding parameters~(both close in proximity to Attention) are also initialized from self-pretrained weights, these results demonstrate that the Attention block is the primary carrier of the benefits of SPT.

To interpret the remaining gap, we note that Tb.~\ref{tab:subsets} uses multilayer architectures. In this setting, pretrained Attention weights may not transfer when the surrounding weights are reset. The Attention map in layer $\ell$ is determined not only by $W_Q$ and $W_K$, but also by representations entering that layer. Re-initializing embeddings, normalizations, or earlier-layer weights changes this representation, including its scale and distribution. Thus, pretrained $W_Q, W_K$ may no longer produce the Attention patterns they learned during SPT. This mismatch can compound across layers, since a distorted Attention pattern in one layer alters the inputs of the next layer. Therefore, Attention-only initialization can recover part of the SPT benefit, but should not be expected to match full SPT in deeper models. 

Motivated by this explanation and by \S~\ref{sec:depth}, we run a final experiment with \textbf{1-layer models}, where such cross-layer mismatch cannot accumulate. We initialize only the parameters that determine the softmax Attention scores, namely $W_Q$ and $W_K$, from SPT weights. Tb.~\ref{tab:one_layer_qk} confirms our interpretation: in 1-layer CIFAR10/PathFinder models, gains can be attributed to a better initialization of $W_Q, W_K$.

\begin{table}[h]
\vspace{-3mm}
\centering
\caption{Please refer to \S~\ref{sec:pre} for our setup. We report mean and standard deviation across 5 seeds. Models are initialized in a \textbf{hybrid fashion}: we consider, for all layers, initializing to SPT values \underline{\textbf{only}} (1) the Attention weights~($W_K, W_Q, W_V, W_O$) (2) the Attention weights, their LayerNorm parameters, and the initial embedding layer, \dots, (5) Attention and MLP weights, and so on. All weights which are not initialized from SPT are initially random. All weights are then optimized with the same setting as Tb.~\ref{tb:reproduce}. ``No Init'' and ``Full Init'' here correspond to the from scratch and SPT variant in Tb.~\ref{tb:reproduce}, respectively. We show classification accuracy, and notice that substantial gains from random initialization can be obtained by initializing to self-pretrained weights only the parameters in the Attention branch, while specialized initialization of the MLP branch has little to no effect.}
\setlength{\tabcolsep}{5pt}
\renewcommand{\arraystretch}{1.15}
\vspace{0.75em}

\resizebox{0.8\textwidth}{!}{%
\begin{tabular}{llcccccc}
\toprule
\multicolumn{8}{c}{\textbf{Baselines and attention-related initialization}} \\
\midrule
\textbf{Dataset} &
\textbf{Stat.} &
\textbf{None} &
\textbf{Full} &
\textbf{Att} &
\textbf{Att+Norm} &
\textbf{Att+Enc} &
\textbf{Att+Norm+Enc} \\
\midrule
CIFAR10 
& mean
& $0.503$
& \cellcolor{blue!15}$0.721$
& $0.532$
& $0.515$
& $0.578$
& \cellcolor{forestgreen!20}\underline{$0.714$} \\
& std.
& $\pm 0.002$
& $\pm 0.012$
& $\pm 0.004$
& $\pm 0.024$
& $\pm 0.009$
& $\pm 0.002$ \\
\midrule
PathFinder
& mean
& $0.672$
& \cellcolor{blue!15} $0.850$
& $0.686$
& $0.677$
& $0.709$
& \cellcolor{forestgreen!20}\underline{$0.781$} \\
& std.
& $\pm 0.003$
& $\pm 0.004$
& $\pm 0.005$
& $\pm 0.007$
& $\pm 0.009$
& $\pm 0.008$ \\
\bottomrule
\end{tabular}%
}

\vspace{0.75em}

\resizebox{0.8\textwidth}{!}{%
\begin{tabular}{llccccc}
\toprule
\multicolumn{7}{c}{\textbf{MLP-related and mixed initialization}} \\
\midrule
\textbf{Dataset} &
\textbf{Stat.} &
\textbf{Att+MLP} &
\textbf{Att+MLP+Enc} &
\textbf{Att+MLP+Norm} &
\textbf{MLP+Norm} &
\textbf{MLP+Norm+Enc} \\
\midrule
CIFAR10 
& mean
& $0.540$
& $0.503$
& $0.504$
& $0.487$
& \cellcolor{red!20}$0.495$ \\
& std.
& $\pm 0.008$
& $\pm 0.006$
& $\pm 0.011$
& $\pm 0.004$
& $\pm 0.009$ \\
\midrule
PathFinder
& mean
& $0.685$
& $0.703$
& $0.661$
& $0.661$
& \cellcolor{red!20}$0.669$ \\
& std.
& $\pm 0.008$
& $\pm 0.005$
& $\pm 0.010$
& $\pm 0.007$
& $\pm 0.005$ \\
\bottomrule
\end{tabular}%
}
\label{tab:subsets}
\vspace{-0mm}
\end{table}

\begin{table*}[h]
\caption{For one-layer models, SPT improves val. accuracy due to better initialization of $W_Q, W_K$. We report mean and standard deviation across 5 seeds. A full ablation is provided in Tb.~\ref{tab:qk-att-ablation}.}
\label{tab:one_layer_qk}
\centering
\begin{small}
\resizebox{\textwidth}{!}{%
\begin{sc}
\begin{tabular}{l l l l l}
\toprule
\textbf{Dataset} & \textbf{From Scratch} & \textbf{SPT $\to$ Full Model init} & \textbf{SPT $\to$ Only QK init} & \textbf{SPT $\to$ All but not QK init} \\
\midrule
CIFAR10   & $0.486 \pm 0.003$ & \cellcolor{blue!15}$0.603 \pm 0.005$  \dgs{+0.12} & \cellcolor{forestgreen!20}$0.607 \pm 0.012 $ \dgs{+0.12} & \cellcolor{red!20}$0.502 \pm 0.008$ \dgs{+0.016}\\
PathFinder & $0.672 \pm 0.002$ & \cellcolor{blue!15}$0.705 \pm 0.004$\dgs{+0.03} & \cellcolor{forestgreen!20}$0.704 \pm 0.003$ \dgs{+0.03} & \cellcolor{red!20}$0.679  \pm 0.007$ \dgs{+0.008}\\
\bottomrule
\end{tabular}
\end{sc}
}
\end{small}
\end{table*}

\mybox[gray!8]{
\textbf{Takeaway 4:} Downstream benefits can be observed even when self-pretraining is applied to only a subset of parameters. In line with Takeaway~3, we find that retaining self-pretrained initialization in the Attention layers is particularly impactful, in contrast to MLP layers. Once depth-related confounding effects are removed, SPT initialization of $W_Q$ and $W_K$ alone is sufficient to explain the performance gains attributed to self-pretraining.
}

\vspace{-1mm}
\paragraph{How much do weights move?}
To complement the mixed-initialization results, we compare Frobenius displacements between random initialization (R), self-pretraining (SPT), supervised training from scratch (SC), and finetuning (FT) (\S~\ref{sec:weight}). Across CIFAR10 and PathFinder, displacements $R\!\rightarrow\!SC$ are generally much smaller than $R\!\rightarrow\!SPT$ or $R\!\rightarrow\!FT$: SC stays closer to initialization. Fig.~\ref{fig:delta_heatmaps} further shows that from-scratch norm changes are dominated by MLP blocks, while Attention projections move less, especially relative to MLPs on PathFinder. This supports the view that SPT provides a useful Attention structure not recovered by supervised training alone.

\section{Understanding self-pretraining on Attention weights}
\label{sec:toy}

In \S~\ref{sec:ablations}, we ablated the Self-Pretraining~(SPT) paradigm on the LRA. We showed that SPT gains are tied to successful optimization~(Takeaway 2), can be observed even in 1-layer, and are especially strong when SPT initializes the Attention weights~(Takeaway 4, Tb.~\ref{tab:subsets} and Tb.~\ref{tab:one_layer_qk}). When randomly initialized and trained with labels, these weights struggle to establish token relations~(Takeaway 3, Tb.~\ref{tab:freezing}). Here, we move beyond LRA and introduce a minimal synthetic task that exposes the benefits of SPT in a controlled setting. This setup enables smaller models and finer ablations, and unlocks insights into how SPT acts on Attention: we find that SPT reshapes how the model uses additive positional encodings and effectively converts them into a flexible proximity bias. 

\vspace{-3mm}
\paragraph{Task and Data.}
We generate a synthetic binary classification dataset of length-100 sequences whose statistics depend on a binary label. For each label, two anchor points define the region where trajectories concentrate: each datapoint is latently an ellipsoidal trajectory moving back and forth between these anchors. To make the task non-trivial, trajectories vary in speed via random-phase time reparametrization, orbital amplitude, and orientation, and are further corrupted by Gaussian noise and outliers. \S~\ref{sec:synthetic_appendix} describes the task and provides example plots. The dataset contains $100$ balanced training sequences with $15\%$ random label flips, $200$ balanced validation sequences, $400$ balanced test sequences, and $12000$ unlabeled sequences for self-supervised pretraining.

\vspace{-2mm}

\paragraph{Model.} We use a one-layer Attention module to predict the label. The module is much simpler than the one used in \S~\ref{sec:ablations}: two-dimensional trajectories are linearly embedded into 32 dimensions and processed by single-head softmax Attention with weights $W_Q, W_K, W_V, W_O$. There is no bias, no MLP after the Attention block, and no LayerNorm. Crucially, we make use of absolute positional embeddings, added to the encoded input before the Attention block exactly as in \S~\ref{sec:ablations}.

\begin{wrapfigure}[17]{r}{0.4\textwidth}
    \centering
    \vspace{-12mm}
    \includegraphics[width=0.98\linewidth]{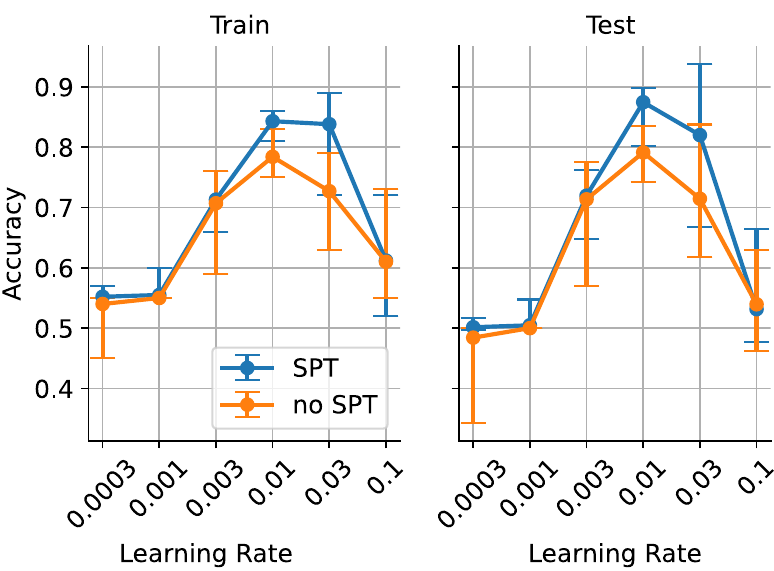}
    \caption{Performance on the toy task (1-layer model) described in \S~\ref{sec:toy}. Train and test accuracies are averaged over 10 random seeds~(max-min interval is shown) for different learning rates. SPT consistently outperforms the no-SPT baseline, achieving higher peak accuracy at intermediate learning rates.}
    \label{fig:toy_main_tuning}
\end{wrapfigure}
\subsection{Results and discussion}
In Fig.~\ref{fig:toy_main_tuning}, we consider 20 epochs of training from labels, preceded by a potential SPT~(10 epochs) using a masked objective~(as in \S~\ref{sec:ablations}). For each learning rate, we gather performance over 10 seeds, and show the average along with the max-min range. \textbf{We observe consistent benefits from self-pretraining in this controlled setting}, with a boost of $\sim 10\%$ in both training and test accuracies at the optimal learning rates. We also provide plots for the training loss and a different epoch budget in \S~\ref{sec:synthetic_appendix}.

\vspace{-3mm}
\paragraph{Attention Matrix.} We proceed to inspect how the Attention matrix evolves from random initialization during pretraining. To do this, we inspect runs associated with the learning rate maximizing the gap in training loss, and then check if similar findings can also extend to the LRA in \S~\ref{sec:back_to_LRA}. In Fig.~\ref{fig:toy_att_evolution}, we see a clear pattern emerging: when feeding a single input into the self-pretrained Attention mechanism, a \textbf{diagonally dominated structure emerges} resembling ALiBi~\citep{press2021train} and thus drawing a direct link to the findings of~\citet{zimerman2023long} on the LRA. Compared to random initialization of $W_Q, W_K$, which produces a near-uniform Attention pattern, as the softmax collapses to a flat distribution~\citet{noci2022signal}, SPT establishes a clear proximity bias: each token attends predominantly to its neighbors.

\begin{figure*}[h]
    \vspace{-1mm}
    \centering
    \includegraphics[width=\linewidth]{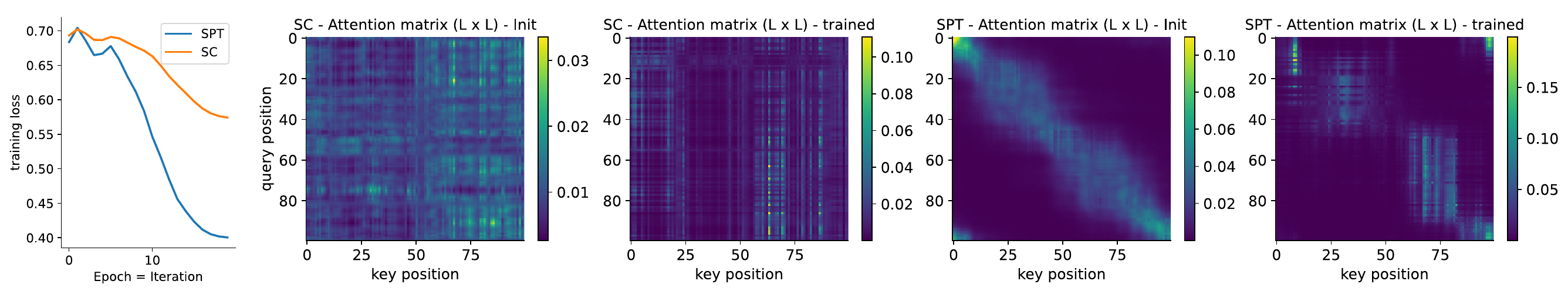}
    \caption{Toy Attention evolution. Left: training loss over iterations for self-pretraining~(SPT) and From-Scratch training~(SC). Right: Attention matrices~($L\times L$) at initialization and after training. SPT rapidly develops structure during pretraining, exhibiting a proximity bias. Compared to random initialization, this structure can develop into a richer sequence mixer after finetuning~(``trained'').}
    \label{fig:toy_att_evolution}
\end{figure*}

\vspace{-3mm}
\paragraph{Crucial role of Positional Encodings.} Motivated by the structure pretraining imposes on the Attention matrix, we inspect how this pattern emerges and its relation to additive positional embeddings. To isolate the role of positional information, we set the input content to zero and feed only positional encodings through the Attention block~(Fig.~\ref{fig:toy_position_undo}). The same diagonal pattern as Fig.~\ref{fig:toy_att_evolution} emerges, indicating a data-independent mechanism. \textbf{SPT Attention weights are in this setting effectively ``undoing'' sinusoidal positional encodings}, building up a structure which instead allows for direct comparisons of nearby tokens already during the first finetuning steps, accelerating convergence.

\begin{figure*}[t]
\vspace{-0mm}
    \centering
    \includegraphics[width=0.95\linewidth]{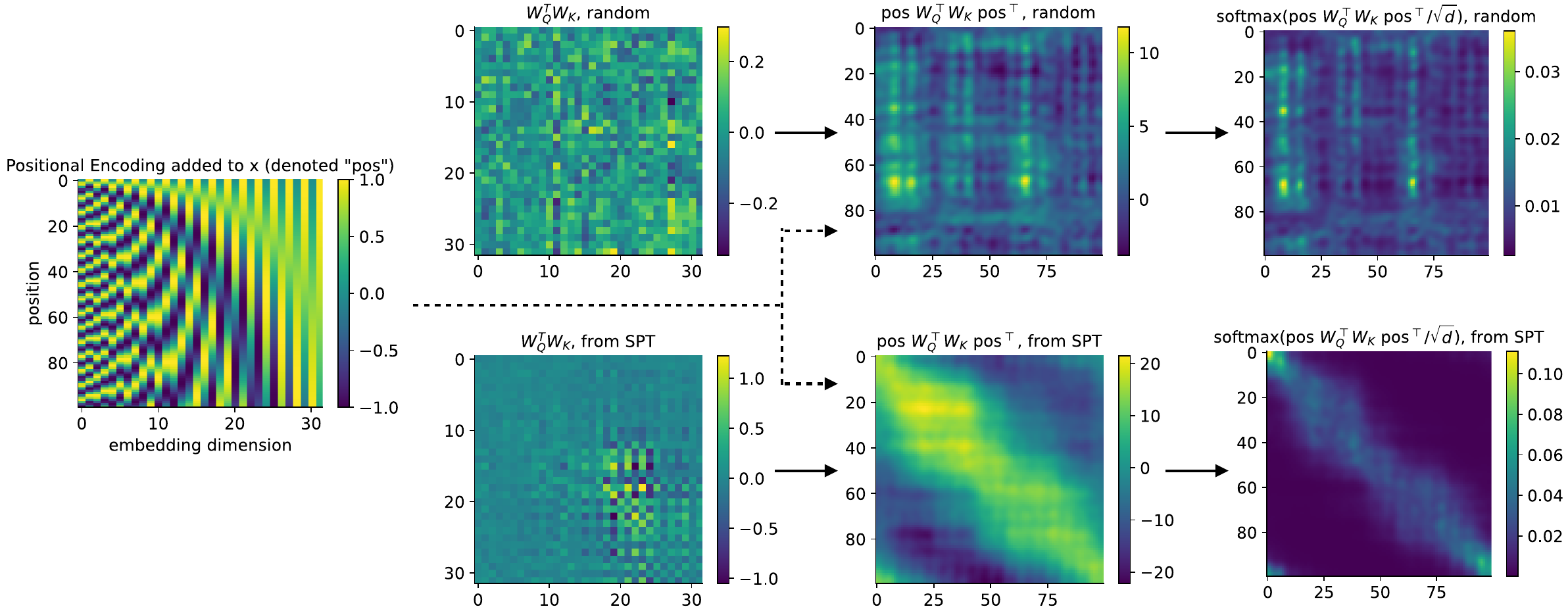}
    \caption{ Visualization of Attention components with random initialization (top) and after SPT training (bottom). While random weights fail to recover positional structure, SPT learns weights that effectively undo positional encoding, producing coherent, position-aligned Attention after softmax. Here we set the input content $X=0$ and feed only positional embeddings $\text{pos}$ through Q/K to isolate the effect of positional information.}
    \label{fig:toy_position_undo}
    \vspace{-3mm}
\end{figure*}

Although the marginal distributions of $W_Q$ and $W_K$ remain close to random initialization in Fig.~\ref{fig:toy_position_undo}, their product $W_Q^\top W_K$ develops a structured interaction with the positional embeddings. As shown in Fig.~\ref{fig:toy_distributions}, the distributions of $W_Q,W_K$ do not exhibit any particular shift from random initialization -- yet their product does, and is able to act on positional embeddings.

\begin{wrapfigure}[24]{r}{0.5\textwidth}
\vspace{-6mm} 
\centering
    \includegraphics[width=\linewidth]{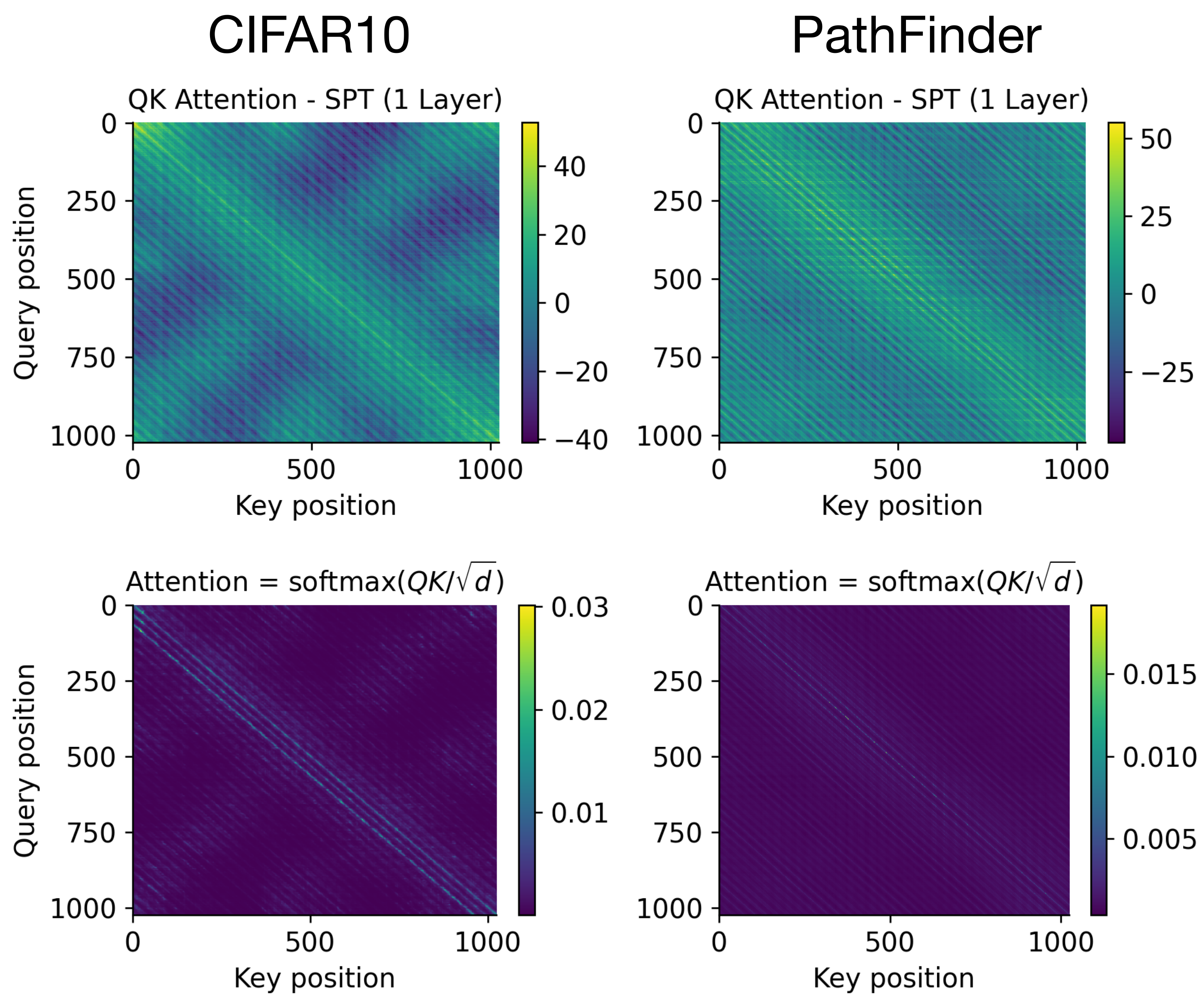}
    \caption{
    Visualization of raw $QK^\top$ (top) and softmax-normalized attention weights (bottom) for a single-layer self-pretrained Transformer. $QK^\top$ matrices show clear diagonal structure, with noisier patterns for CIFAR10 and more coherent structure for PathFinder. Softmax normalization yields sparse, predominantly diagonal attention. Weights are taken after SPT (before finetuning) for the 1-layer setting described in Tb.~\ref{tab:one_layer_qk}.
    }
    \label{fig:att_pattern_cifar_pt}
\end{wrapfigure}

\subsection{Bringing back insights to the LRA}
\label{sec:back_to_LRA}
The toy task in the previous subsection suggests self-pretraining effectively builds a diagonally-dominated Attention structure~(Fig.~\ref{fig:toy_att_evolution}) through a direct action on the positional encodings controlled by the product $W_Q^\top W_K$~(Fig.~\ref{fig:toy_position_undo}). We next ask whether the same pattern appears on the LRA. To do this, we return to the simplest setup in \S~\ref{sec:ablations}, i.e., the one of Tb.~\ref{tab:one_layer_qk}.

\vspace{-3mm}
\paragraph{Attention Structure.} Fig.~\ref{fig:att_pattern_cifar_pt} shows the query-key product and the resulting softmax Attention matrix obtained after feeding a CIFAR10 or a PathFinder example to the self-pretrained network. The resulting pattern is strikingly similar to Fig.~\ref{fig:toy_att_evolution}: pretrained $W_Q, W_K$ establish a diagonal-dominated Attention.

\vspace{-3mm}
\paragraph{Different PE.} At this point, a natural question emerges: rather than relying on SPT to bend $W_Q, W_K$ into producing a proximity bias, why not impose such a bias directly at initialization? 
Instead of using absolute positional encodings, one can inject positional structure directly into Attention: writing $S$ as the softmax argument, ALiBi~\citep{press2021train} sets $S_{ij}\mapsto S_{ij}-m|i-j|$, hard-coding a preference for nearby tokens. FIRE~\citep{li2023functional} learns the shape of the bias $S_{ij}\mapsto S_{ij}+\psi_\theta(|i-j|)$. RoPE~\citep{su2024roformer} instead modifies the dot product itself by rotating queries and keys, making the score depend on relative position $j-i$ without adding an explicit bias. Tb.~\ref{tab:pe-ALiBi} shows that ALiBi can improve from-scratch trainability on CIFAR10, but it does not unlock PathFinder performance, consistent with~\citet{zimerman2023long}. The effect is not specific to ALiBi, suggesting that the bottleneck is not access to the right positional information, but learning the right Attention bias. CIFAR10 can benefit from simple locality priors since nearby visual statistics are preserved after flattening. PathFinder instead requires tracing global connectivity; deciding whether the points are connected in 2D can require much longer-range sequence interactions. A fixed or poorly learned bias can prevent the model from forming the necessary global connections. SPT is the only strategy we found to reliably unlock this regime. We return to why this is the case in the next section.


\begin{table*}[t]
\vspace{-0mm}
\caption{Same setup as Tb.~\ref{tb:reproduce}, ablating the effect of positional encoding variants on self-pretraining accuracy. We report from-scratch (FS) and self-pretraining (SPT) accuracy on a single run.}
\label{tab:pe-ALiBi}
\vspace{-2mm}
\begin{center}
\begin{small}
\resizebox{0.95\textwidth}{!}{%
\begin{sc}
\begin{tabular}{l l c c c c c c c}
\toprule
\textbf{Dataset} & \textbf{Strategy}
& \textbf{PE}
& \textbf{ALiBi}
& \textbf{FIRE}
& \textbf{RoPE}
& \textbf{PE+ALiBi}
& \textbf{PE+FIRE}
& \textbf{RoPE+ALiBi} \\
\midrule
\multirow{2}{*}{CIFAR10}
& FS
& 0.501 & 0.564 & 0.594 & 0.487 & 0.588 & 0.596 & 0.559 \\
& SPT
& 0.722 & 0.575 & 0.746 & 0.741 & 0.754 & 0.760 & 0.759 \\
\midrule
\multirow{2}{*}{PathFinder}
& FS
& 0.671 & 0.500 & 0.500 & 0.679 & 0.678 & 0.674 & 0.661 \\
& SPT
& 0.8506 & 0.500 & 0.500 & 0.71 & 0.9007 & 0.714 & 0.782 \\
\bottomrule
\end{tabular}
\end{sc}
}
\end{small}
\end{center}
\vspace{-4mm}
\end{table*}


\section{Why label supervision can fail to learn Attention maps}
\label{sec:theory}
In \S~\ref{sec:ablations} we presented a sequence of ablations showing that SPT is particularly helpful for learning attention weights. In~\S~\ref{sec:toy}, we showed that SPT induces learning of a proximity bias through matrices $W_Q$ and $W_K$, moving from random uniform attention scores to a diagonally-dominated token mixing. The experiments show what SPT-learned Attention maps look like, but not why label supervision may fail to learn them. We now give a simplified theoretical explanation.\\
Consider Attention magnitudes $a_{ij}=\exp(s_{ij}) / \sum_{k=1}^L\exp(s_{ik})$, parametrized by scores $s_{ij}$. We assume $s_{ij}$ to be independent of the input, an abstraction motivated by~Fig.~\ref{fig:toy_position_undo} and~\ref{fig:att_pattern_cifar_pt}. We consider learning a target pattern $s^*_{ij}=\Delta_{ij}$. To simplify the analysis, we study one-dimensional perturbations: $s_{ij}(\alpha)=\alpha \Delta_{ij}$. The solution is $\alpha^*=1$, while $\alpha=0$ models uniform Attention initialization. 

\begin{proposition} Let $X_1,\ldots,X_L\in\mathbb R^d$ be centered random tokens with finite second moments, and define the autocorrelation matrix $C\in\mathbb R^{L\times L}$ as $C_{ij}:=\mathbb E[X_i^\top X_j]$. For a score pattern $\Delta\in\mathbb R^{L\times L}$, consider training a single parameter $\alpha$ controlling the Attention score $s_{ij}(\alpha)$. Outputs are computed as $o_i(\alpha)=\sum_j a_{ij}(\alpha)X_j$. For a label $Y$, consider the mean-pooled~(from scratch, representations are aggregated) loss $\mathcal L_{\rm sup}(\alpha)=\mathbb E[\ell(\frac1{L}\sum_i o_i(\alpha),Y)]$ and the SPT loss $\mathcal L_{\rm SPT}(\alpha)=\frac1{2L}\sum_i\mathbb E[\|X_i-o_i(\alpha)\|^2]$. Let $\mathbf 1\in\mathbb R^L$ be the all-ones vector and define $\mathcal B:=\{\Gamma\in\mathbb R^{L\times L}: \Gamma^\top\mathbf 1=0\}$. Define $\langle A,B\rangle_F=\sum_{i,j}A_{ij}B_{ij}$ and $\bar \Delta = \Delta-\frac{1}{L}\Delta\mathbf 1\mathbf 1^\top$. For every $\Delta\in\mathcal B$, 
$$
\mathcal L_{\rm sup}'(0)=0,
\qquad
\mathcal L_{\rm SPT}'(0)=-\tfrac1{L^2}\langle \bar \Delta,C\rangle_F. 
$$
An analogous statement holds for the case where token-to-same-token interactions are excluded.
\label{prop:theory1}
\end{proposition}
\vspace{-2mm}

The proposition, proven in~\S\ref{sec:proofs}, shows, in a simplified setup, that there exist Attention patterns that are locally invisible to label supervision at uniform initialization under a mean-pooled loss. Instead, such patterns are learnable through the SPT loss. Fig.~\ref{fig:spt_theory} verifies the result numerically for a locality-inducing pattern $\Delta$. This supports the interpretation that reconstruction losses can provide stronger gradients for Attention-score directions that mean-pooled label supervision cannot detect locally.

\section{Conclusions}
\label{sec:conclusion}
We studied self-pretraining (SPT) for Transformer sequence classification with the goal of understanding why masked can improve supervised training. Across LRA ablations and a controlled synthetic task, our evidence points to Attention-score learning as a central bottleneck: label training often fails to organize useful query-key interactions. In our toy task and in one-layer CIFAR10/PathFinder models, we identified that SPT provides a flexible, proximity-biased Attention structure that fine-tuning can build on. Our theory provides one reason this can happen, showing that mean-pooled supervision can be locally blind to certain directions in the Attention scores that masked reconstruction can detect.

\textbf{Limitations.} The most detailed multi-seed component ablations focus on CIFAR10 and PathFinder; PathX is excluded for compute reasons, and several large LRA runs are single-seed. The toy task and theory deliberately simplify the architecture, data distribution, and Attention parametrization. Other effects, including LayerNorm, embeddings, depth, optimizer dynamics, and dataset-specific structure, may also contribute in larger settings. The theory isolates one possible obstruction, not the full explanation. We believe a more general theory is possible and of great interest, perhaps incorporating also empirical evidence from other aggregation strategies, including a CLS token.

\section*{Acknowledgments}
Antonio Orvieto acknowledges the financial support of the Hector Foundation and the computing resources provided by MPI-IS.

Coser Omar is a Ph.D. student enrolled in the National Ph.D. in Artificial Intelligence, XXXVIII cycle, course on Health and life sciences, organized by Università Campus Bio-Medico di Roma.
\bibliographystyle{plainnat}  
\bibliography{references} 

\clearpage
\newpage
\appendix
\section{Setup}
We describe below the LRA tasks, models, and experimental details.
\subsection{Tasks}
\label{sec:tasks}
\begin{itemize}
    \item \textbf{ListOps:} Inspired by~\citep{nangia2018listops}. Sequences are nested lists describing mathematical operations applied to token sets. The task requires understanding hierarchical structures. The sequence length is 2048.
    \item \textbf{Text:} Character-level IMDb reviews~\citep{maas2011imdb} for binary sentiment classification. Sequences are up to 2048 tokens.
    \item \textbf{Retrieval:} Character-level binary classification of document similarity scores with sequences up to 4096 tokens. The task was introduced by~\citet{radev2013acl}.
    \item \textbf{Image:} Grayscale CIFAR10~\citep{krizhevsky2009learning} images flattened into 1D sequences for 10-way classification without explicit 2D inductive bias, sequence length $1024$.
    \item \textbf{PathFinder, PathX:} Synthetic 2D visual tasks introduced by~\citep{linsley2018learning}, treated as 1D sequences for tracing capabilities (sequence lengths 1024 and 16384, respectively). They are both binary classification tasks.
\end{itemize}

\subsection{Experimental details}
\label{sec:exp_details}

Our training settings in most ablations closely follow choices by~\citet{amos2023never}.

\begin{itemize}
    \item For \textbf{ListOps}, the model is trained with an embedding size of 512, passed through 6 Transformer layers, each with 8 Attention heads. The feed-forward network has a hidden size of 1024. The tuned learning rates used for our experiments are $1 \times 10^{-4}$ (finetuning) and $1 \times 10^{-3}$ (pre-training), with batch sizes 64 to 128~(respectively). Weight decay is set at $0.1$, and cross-entropy is used as the pretraining loss. 
    
    \item In the \textbf{Text} dataset, the training strategy remains similar to ListOps, with identical feature size, depth, and Attention heads. However, the learning rates are slightly adjusted to $1 \times 10^{-4}$ (finetuning) and $5 \times 10^{-4}$ (pre-training), with batch sizes of 64 and 32~(respectively).
    
    \item For the \textbf{Retrieval} dataset, the model adopts a smaller feature size of 128 and a reduced depth of 4 layers, with 4 Attention heads and a feed-forward hidden size of 512. The learning rates are $5 \times 10^{-4}$ (finetuning) and $5 \times 10^{-3}$ (pre-training), with batch sizes of 16 and 32 respectively. Notably, weight decay is here set to zero, and cross-entropy loss is used. 
    
    \item The Image dataset~(a.k.a. \textbf{CIFAR10}) employs a distinct setup, with a feature size of 64 and a shallow depth of 3 layers. There are 4 Attention heads and a feed-forward hidden size of 128. Mean pooling is used at the last layer, reflecting the spatial nature of image data. The learning rates are fixed at $1 \times 10^{-3}$ (finetuning and pre-training), with batch sizes of 16 and 32~(finetuning and pre-training, respectively). Weight decay is set to zero. The pretraining loss is L2 loss, emphasizing pixel-level reconstruction rather than classification.
    
    \item Finally, in the \textbf{PathFinder} dataset, the model configuration closely resembles that of Retrieval, with a feature size of 128, a depth of 4 layers, 4 Attention heads, and a feed-forward size of 512. The learning rates vary between $5 \times 10^{-4}$ (finetuning) and $1 \times 10^{-3}$ (pre-training), with batch sizes of 16 and 32~(respectively). Like ListOps and Retrieval, cross-entropy loss is used for pretraining.
\end{itemize}
During further experiments involving changing the number of layers and datasets, we found these settings to be mostly stable, with only minor adjustments to the learning rate needed.


\subsection{Model}
\label{sec:model_appendix}

We always use a pre-norm Transformer model. Given a minibatch
\[
    \mathcal B=\{(s_b,y_b)\}_{b=1}^B,\qquad
    s_b=(s_{b,1},\ldots,s_{b,\ell_b}),\qquad \ell_b\le L,
\]
each input element is mapped to
\(X_{b,t}\in\mathbb R^D\). For discrete or quantized inputs,
\(X_{b,t}=E(s_{b,t})\); for continuous inputs,
\(X_{b,t}=E(x^{\rm raw}_{b,t})\), where \(E\) is a learned linear encoder.
Sinusoidal positional encodings are added once at the input:
\[
    Z^{(0)}_{b,t}=\mathrm{Drop}(X_{b,t}+P_t).
\]

For layer \(\ell=0,\ldots,T-1\), we write \(Z=Z^{(\ell)}\) and
\(\widetilde Z=\mathrm{LN}(Z)\). With \(H\) heads and \(d=D/H\), the head
projections are
\[
    W_Q^{\ell,h},W_K^{\ell,h},W_V^{\ell,h}\in\mathbb R^{d\times D},
\]
and
\[
    Q_h=\widetilde Z(W_Q^{\ell,h})^\top,\qquad
    K_h=\widetilde Z(W_K^{\ell,h})^\top,\qquad
    V_h=\widetilde Z(W_V^{\ell,h})^\top .
\]
Hence the raw score matrix for head \(h\) is
\[
    S_h
    =
    \frac{Q_hK_h^\top}{\sqrt d}
    =
    \frac{\widetilde Z (W_Q^{\ell,h})^\top W_K^{\ell,h}\widetilde Z^\top}{\sqrt d}.
\]
We therefore interpret \((W_Q^{\ell,h})^\top W_K^{\ell,h}\in\mathbb R^{D\times D}\).
The row-wise attention weights and multi-head output are
\[
    A_h=\mathrm{softmax}_{\rm row}(S_h),\qquad
    \mathrm{MHA}_{\ell}(\widetilde Z)
    =
    \mathrm{Concat}_{h=1}^H(A_hV_h)W_O^\ell .
\]
The residual updates are
\[
    Z' = Z+\mathrm{Drop}(\mathrm{MHA}_{\ell}(\mathrm{LN}(Z))),
\]
\[
    Z^{(\ell+1)}
    =
    Z'
    +
    \mathrm{Drop}\!\left(
        \mathrm{MLP}_{\ell}(\mathrm{LN}(Z'))
    \right),
\]
where
\[
    \mathrm{MLP}_{\ell}(x)=W_2^\ell\sigma(W_1^\ell x+b_1^\ell)+b_2^\ell .
\]

For supervised classification, the final sequence is mean-pooled over
non-padding positions:
\[
    h_b=\frac{1}{\ell_b}\sum_{t=1}^{\ell_b} Z^{(T)}_{b,t,:},
    \qquad
    \hat y_b=W_{\rm cls}h_b+b_{\rm cls},
\]
and the model is trained with cross-entropy against \(y_b\).

For self-pretraining, a mask set
\(\mathcal M_b\subseteq\{1,\ldots,\ell_b\}\) is sampled for each sequence.
The encoder receives the masked input, while the target is the original
unmasked input element
\[
    u_{b,t}=
    \begin{cases}
        x^{\rm raw}_{b,t}, & \text{continuous inputs},\\
        s_{b,t}, & \text{discrete or quantized inputs}.
    \end{cases}
\]
A token-wise head predicts
\[
    r_{b,t}=W_{\rm spt}Z^{(T)}_{b,t,:}+b_{\rm spt},
\]
where the output dimension is the raw input dimension for continuous
reconstruction and the vocabulary or quantization size for discrete
reconstruction. The SPT objective is
\[
    \mathcal L_{\rm SPT}
    =
    \frac{1}{\sum_b|\mathcal M_b|}
    \sum_{b=1}^B
    \sum_{t\in\mathcal M_b}
    \ell_{\rm SPT}^{(m)}(r_{b,t},u_{b,t}),
\]
with
\[
    \ell_{\rm SPT}^{(m)}(r,u)
    =
    \begin{cases}
        \frac12\|r-u\|_2^2,
        & \text{continuous targets},\\
        -\log \mathrm{softmax}(r)_u,
        & \text{discrete or quantized targets}.
    \end{cases}
\]
Thus SPT reconstructs the original input element, not the embedded vector \(X_{b,t}\).

\clearpage
\newpage
\section{Further Experimental results}

\subsection{SPT evidence beyond the LRA}
\label{sec:other_datasets}
We chose four datasets in \cite{dempster2019rocket} to showcase that SPT also works beyond the LRA, as suggested from earlier investigations~\citep{el2021large}. GunPoint and GunPointAgeSpan are hand-motion datasets. ECG200 and ECGFiveDays are electrocardiogram datasets. We ran each experiment on five random seeds, reporting mean plus standard deviation.
\begin{table}[h]
\centering
\caption{Validation accuracy on UCR datasets: training from scratch vs.\ self-pretraining (SPT). Higher is better; bold indicates the best result per row.}
\label{tab:spt-results}
\begin{tabular}{lcc}
\toprule
\textbf{Dataset} & \textbf{From Scratch} & \textbf{SPT} \\
\midrule
GunPoint        & $0.914 \pm 0.024$ & $0.964 \pm 0.013$ \\
GunPointAgeSpan & $0.923 \pm 0.026$ & $0.955 \pm 0.011$  \\
ECG200          & $0.919 \pm 0.019$ & $0.954 \pm 0.030$  \\
ECG5Days        & $0.926 \pm 0.008$ & $0.931 \pm 0.011$ \\
\midrule
Average & $0.921 \pm 0.019$ & $0.951 \pm 0.016$ \\
\bottomrule
\end{tabular}
\end{table}

\subsection{Precise numerical results for Fig.~\ref{fig:pt_epochs}}
\begin{table}[h]
\caption{Ablation on number of epochs, see \S~\ref{sec:duration}. These data are shown in Fig.~\ref{fig:pt_epochs}}
\label{tb:layers}
\vskip 0.15in
\begin{center}
\begin{small}
\begin{sc}
        \begin{tabular}{rlll}
            \toprule
            \textbf{Pretr. Len.} & ListOps & CIFAR10 & PathFinder \\
            \midrule
            0 epochs    & 0.403      & 0.499       & 0.67 \\
            1 epoch   & 0.48\dgs{+0.08}  & 0.56\dgs{+0.06}  & 0.69\dgs{+0.02} \\
            10 epochs   & 0.55\dgs{+0.15}  & 0.71\dgs{+0.21}  & 0.71\dgs{+0.04} \\
            50 epochs   & 0.55\dgs{+0.15}  & 0.71\dgs{+0.21}  & 0.71\dgs{+0.04} \\
            100 epochs  & 0.55\dgs{+0.15}  & 0.71\dgs{+0.21}  & 0.73\dgs{+0.06} \\
            150 epochs  & 0.56\dgs{+0.16}  & 0.71\dgs{+0.21}  & 0.77\dgs{+0.10} \\
            200 epochs  & 0.56\dgs{+0.16}  & 0.72\dgs{+0.22}  & 0.85\dgs{+0.18} \\
            \bottomrule
        \end{tabular}
\end{sc}
\end{small}
\end{center}
\vskip -0.1in
\end{table}

\subsection{Precise numerical results for Fig.~\ref{fig:depth}}
\begin{table}[h]
\caption{Ablation over the number of layers; numerical values for Fig.~\ref{fig:depth}.}

\label{tb:app_epoch_ablation}
\vskip 0.15in
\begin{center}
\begin{small}
\begin{sc}
\begin{tabular}{@{}r|lllllr@{}}
\toprule
\textbf{Data }&\textbf{Modality} & \textbf{1 layer} & \textbf{2 layers}& \textbf{3/4 layers} \\
\hline
& From scratch (train) & 0.69 &0.78 &0.88 \\
&From scratch (val)  & 0.48& 0.48& 0.50 \\
CIFAR 10 &+ Self Pre-trained (train) &0.93 & 0.95&0.96\\
 &+ Self Pre-trained (val) &0.60 \dgs{+0.13} &0.67 \dgs{+0.19}&0.72 \dgs{+0.22}\\
&Pre-trained w/ PathFinder (train) & 0.84 &0.94 &0.97  \\
&Pre-trained w/ PathFinder (val) & 0.56 \dgs{+0.08}& 0.56 \dgs{+0.08}& 0.57 \dgs{+0.07} \\
\hline
& From scratch (train) & 0.67 & 0.76 & 0.83 \\
&From scratch (val) &0.669 & 0.670 & 0.672\\
PathFinder &Self Pre-trained (train) &0.77 &0.85 & 0.97\\
 &Self Pre-trained (val) &0.70 \dgs{+0.03}&0.78 \dgs{+0.10}&0.85 \dgs{+0.18}\\
&Pre-trained w/ CIFAR 10 (train) & 0.85& 0.89&0.91\\
&Pre-trained w/ CIFAR 10 (val) & 0.68 \dgs{+0.01}& 0.68 \dgs{+0.01}& 0.68 \dgs{+0.01}\\
\bottomrule
\end{tabular}
\end{sc}
\end{small}
\end{center}
\vskip -0.1in
\end{table}

\subsection{Full Ablation on CIFAR and PathFinder 1 Layer}
The minimal setup needed to recover full SPT initialization performance in $1$ layer is to initialize only the query and key (QK) projections. On CIFAR10, QK+Enc+Norm performs best, improving over full initialization by about +3 percentage points, while Att+Enc+Norm yields a smaller gain. Most other strategies provide only marginal improvements, except Att+MLP+Norm, which matches the Only Attention configuration. On PathFinder, the trend is similar but more nuanced: QK+Norm, Att+MLP+Norm, and Att+Enc+Norm slightly surpass full initialization, while Only QK nearly matches it. Most remaining configurations offer little benefit over training from scratch. 

\begin{table*}[ht]
\caption{One-layer models~(extended Tb.~\ref{tab:one_layer_qk}): ablation of SPT initialization strategies focusing on Attention and QK parameters. See full discussion in~\S\ref{sec:subset}.}
\label{tab:qk-att-ablation}
\vspace{2mm}
\centering
\scriptsize
\setlength{\tabcolsep}{3pt}
\resizebox{\textwidth}{!}{
\begin{tabular}{lccccccccc}
\toprule
\textbf{Dataset} &
\textbf{Full Init} &
\textbf{Only Att} &
\textbf{Only QK} &
\textbf{QK+Enc+Norm} &
\textbf{Att+Enc+Norm} &
\textbf{QK+Norm} &
\textbf{From Scratch} &
\textbf{Lin+Att+MLP} &
\textbf{Att+MLP+Norm} \\
\midrule
CIFAR10 &
0.603 & 0.598 & 0.607 & 0.64 & 0.62 & 0.521 & 0.486 & 0.540 & 0.598 \\
PathFinder &
0.706 & 0.699 & 0.703 & 0.701 & 0.708 & 0.7105 & 0.671 & 0.687 & 0.71 \\
\bottomrule
\end{tabular}
}
\vspace{-3mm}
\end{table*}

\clearpage
\newpage
\subsection{Weights Analysis}
\label{sec:weight}

In the following tables, $R$ denotes random initialization, SPT denotes the checkpoint after masked reconstruction, SC denotes the checkpoint obtained by supervised training from scratch, and FT denotes the checkpoint obtained by finetuning the SPT model on labels. For a parameter block $\theta$, each entry $A\!\rightarrow\!B$ reports the Frobenius displacement $\|\theta_B-\theta_A\|_F$, while the final columns report the corresponding parameter norms. These measurements separate actual training trajectories, such as $R\!\rightarrow\!SC$, $R\!\rightarrow\!SPT$, and $SPT\!\rightarrow\!FT$, from endpoint comparisons such as $SC\!\rightarrow\!FT$. They are intended to quantify whether supervised training from random initialization substantially reorganizes the same parameter blocks that are reshaped by the SPT pipeline.

Let $Q_c^{(s)}$ denote the (flattened) weight tensor of component $c$ at training stage
$s \in \{\mathrm{R}, \mathrm{SPT}, \mathrm{SC}, \mathrm{FT}\}$. We report the parameter
displacement $\|Q_c^{(s_2)} - Q_c^{(s_1)}\|$ between stages and the parameter norm
$\|Q_c^{(s)}\|$, both computed as $\ell_2$ (Frobenius) norms.

\subsubsection{Analysis for CIFAR}

\begin{table*}[ht]
\caption{Encoder input parameter displacement and norms across training stages.
The top row reports the aggregate encoder statistics, while subsequent rows
provide a layer- and component-wise breakdown.
The difference between the weight norms obtained from training from scratch
and random initialization is small for the encoder
($\Delta_{\text{norm}} = \|Q_{SC}\| - \|Q_R\| = -0.25$),
indicating that training from scratch does not significantly alter the scale
of the encoder input parameters relative to random initialization.}
\centering
\small
\setlength{\tabcolsep}{4pt}
\resizebox{\textwidth}{!}{%
\begin{tabular}{llrrrrrrrrr}
\toprule
Layer & Component &
R$\rightarrow$SPT &
R$\rightarrow$SC &
SPT$\rightarrow$FT &
SC$\rightarrow$FT &
R$\rightarrow$FT &
$\|Q_R\|$ &
$\|Q_{SPT}\|$ &
$\|Q_{SC}\|$ &
$\|Q_{FT}\|$ \\
\midrule
\multicolumn{2}{l}{\textbf{Encoder}}
& 7.94 & 0.66 & 12.22 & 18.64 & 18.75 & 4.39 & 5.84 & 4.14 & 17.34 \\
\midrule
0 & Attn Q & 24.51 & 7.90 & 49.55 & 59.50 & 59.43 & 4.66 & 25.48 & 9.51 & 60.22 \\
0 & Attn K & 25.87 & 8.86 & 77.88 & 90.54 & 90.73 & 4.66 & 26.98 & 10.31 & 91.58 \\
0 & Attn V & 11.61 & 5.35 & 14.16 & 10.98 & 9.44 & 4.66 & 12.07 & 7.32 & 8.20 \\
0 & Attn O & 12.16 & 5.08 & 20.15 & 19.81 & 19.26 & 4.66 & 12.01 & 6.35 & 18.93 \\
0 & Norm   & 5.76  & 2.71 & 3.84  & 5.87  & 6.75  & 8.00 & 7.20 & 6.95 & 4.80 \\
\midrule
1 & MLP Up   & 28.41 & 11.01 & 76.04 & 84.73 & 84.34 & 6.60 & 29.30 & 13.44 & 84.98 \\
1 & MLP Down & 26.96 & 9.98  & 96.68 & 107.65& 107.17& 4.62 & 27.38 & 11.36 & 107.41 \\
1 & MLP All  & 39.17 & 14.86 & 123.00& 136.99& 136.38& 8.05 & 40.10 & 17.60 & 136.97 \\
1 & Norm     & 4.05  & 1.59  & 6.86  & 8.19  & 8.63  & 8.00 & 6.64 & 7.61 & 11.24 \\
\midrule
2 & Attn Q & 20.61 & 7.79 & 49.30 & 60.86 & 60.76 & 4.67 & 21.68 & 9.46 & 61.84 \\
2 & Attn K & 22.47 & 8.07 & 59.87 & 74.09 & 74.11 & 4.67 & 23.65 & 9.85 & 75.14 \\
2 & Attn V & 17.56 & 5.97 & 50.13 & 59.42 & 59.18 & 4.67 & 18.06 & 7.86 & 59.67 \\
2 & Attn O & 15.83 & 5.58 & 58.62 & 65.75 & 65.54 & 4.67 & 16.35 & 7.51 & 66.03 \\
2 & Norm   & 3.56  & 1.77 & 3.35  & 4.57  & 5.34  & 8.00 & 6.34 & 7.17 & 4.56 \\
\midrule
3 & MLP Up   & 24.95 & 11.82 & 102.90 & 109.50 & 109.44 & 6.63 & 25.25 & 14.24 & 109.73 \\
3 & MLP Down & 27.74 & 11.46 & 111.26 & 118.17 & 117.38 & 4.65 & 27.87 & 12.85 & 117.45 \\
3 & MLP All  & 37.31 & 16.46 & 151.55 & 161.10 & 160.49 & 8.10 & 37.60 & 19.18 & 160.73 \\
3 & Norm     & 3.62  & 1.30 & 6.96   & 6.26   & 6.18   & 8.00 & 6.19 & 8.33 & 11.70 \\
\midrule
4 & Attn Q & 18.96 & 9.08 & 48.32 & 53.14 & 52.66 & 4.73 & 19.88 & 10.67 & 53.24 \\
4 & Attn K & 19.16 & 8.84 & 57.86 & 64.84 & 64.36 & 4.73 & 20.23 & 10.53 & 64.99 \\
4 & Attn V & 18.59 & 6.30 & 72.57 & 79.92 & 79.89 & 4.73 & 19.36 & 8.30 & 80.34 \\
4 & Attn O & 18.60 & 5.60 & 82.62 & 89.12 & 88.90 & 4.73 & 19.47 & 7.90 & 89.39 \\
4 & Norm   & 3.07  & 0.97 & 3.66  & 4.53  & 4.65  & 8.00 & 7.50 & 7.85 & 4.50 \\
\midrule
5 & MLP Up   & 24.12 & 12.93 & 125.84 & 131.78 & 131.52 & 6.58 & 24.53 & 15.03 & 131.73 \\
5 & MLP Down & 31.98 & 13.98 & 135.84 & 143.46 & 142.55 & 4.64 & 32.28 & 15.15 & 142.70 \\
5 & MLP All  & 40.06 & 19.04 & 185.17 & 194.80 & 193.96 & 8.05 & 40.54 & 21.34 & 194.21 \\
5 & Norm     & 2.88  & 1.63 & 18.84  & 17.62  & 18.48  & 8.00 & 7.12 & 8.96 & 25.47 \\
\bottomrule
\end{tabular}%
}
\label{tab:encoder_displacement_combined}
\end{table*}

\subsubsection{Analysis for PathFinder}
\begin{table*}[ht]
\caption{Encoder input parameter displacement and norms across training stages.
The top row reports the aggregate encoder statistics, while subsequent rows
provide a layer- and component-wise breakdown.
The difference between the weight norms obtained from training from scratch
and random initialization is small for the encoder
($\Delta_{\text{norm}} = \|Q_{SC}\| - \|Q_R\| = -0.16$),
indicating that training from scratch does not significantly alter the scale
of the encoder input parameters relative to random initialization.}
\centering
\footnotesize
\setlength{\tabcolsep}{4pt}
\resizebox{\textwidth}{!}{%
\begin{tabular}{llrrrrrrrrr}
\toprule
Layer & Component &
R$\rightarrow$SPT &
R$\rightarrow$SC &
SPT$\rightarrow$FT &
SC$\rightarrow$FT &
R$\rightarrow$FT &
$\|Q_R\|$ &
$\|Q_{SPT}\|$ &
$\|Q_{SC}\|$ &
$\|Q_{FT}\|$ \\
\midrule
\multicolumn{2}{l}{\textbf{Encoder}}
& 17.96 & 3.57 & 1.85 & 19.47 & 18.92 & 6.70 & 16.90 & 6.54 & 18.05 \\
\midrule
0 & Attn Q & 127.64 & 48.66 & 54.24 & 138.04 & 129.23 & 6.57 & 127.76 & 49.11 & 129.32 \\
0 & Attn K & 172.04 & 50.52 & 54.30 & 185.61 & 178.34 & 6.57 & 172.33 & 50.99 & 178.60 \\
0 & Attn V & 15.42  & 52.44 & 11.64 & 55.59  & 18.32  & 6.57 & 13.99  & 53.05 & 17.15  \\
0 & Attn O & 98.96  & 62.65 & 32.75 & 116.54 & 97.86  & 6.57 & 98.83  & 63.49 & 97.73  \\
0 & Norm   & 10.69  & 6.69  & 0.94  & 11.23  & 10.59  & 11.31& 4.70   & 11.86 & 4.34   \\
\midrule
1 & MLP Up   & 108.61 & 53.88 & 46.07 & 128.11 & 116.72 & 6.57 & 108.93 & 54.81 & 117.05 \\
1 & MLP Down & 173.71 & 47.34 & 54.66 & 189.87 & 184.05 & 6.59 & 174.00 & 48.07 & 184.33 \\
1 & MLP All  & 204.87 & 71.73 & 71.48 & 229.05 & 217.95 & 9.30 & 205.28 & 72.90 & 218.36 \\
1 & Norm     & 10.30  & 5.39  & 3.68  & 10.90  & 9.93   & 11.31& 9.45   & 11.96 & 11.07  \\
\midrule
2 & Attn Q & 137.44 & 50.44 & 51.81 & 146.12 & 137.71 & 6.57 & 137.67 & 50.85 & 137.94 \\
2 & Attn K & 159.23 & 48.64 & 49.88 & 167.84 & 160.89 & 6.57 & 159.67 & 49.10 & 161.34 \\
2 & Attn V & 103.28 & 31.03 & 53.13 & 123.11 & 119.14 & 6.57 & 103.49 & 31.64 & 119.41 \\
2 & Attn O & 108.03 & 33.27 & 57.09 & 126.12 & 122.11 & 6.57 & 108.13 & 33.79 & 122.23 \\
2 & Norm   & 7.82   & 7.31  & 2.68  & 7.65   & 8.52   & 11.31& 8.35   & 8.03  & 6.86   \\
\midrule
3 & MLP Up   & 121.88 & 48.94 & 60.51 & 141.23 & 133.06 & 6.54 & 121.89 & 49.51 & 133.12 \\
3 & MLP Down & 189.59 & 49.10 & 70.81 & 205.75 & 199.57 & 6.60 & 189.76 & 49.59 & 199.72 \\
3 & MLP All  & 225.39 & 69.33 & 93.14 & 249.56 & 239.86 & 9.29 & 225.54 & 70.07 & 240.02 \\
3 & Norm     & 7.30   & 6.25  & 2.13  & 7.22   & 6.64   & 11.31& 9.03   & 9.73  & 8.99   \\
\midrule
4 & Attn Q & 89.88  & 48.52 & 50.80 & 108.37 & 97.04  & 6.56 & 90.05  & 49.04 & 97.17  \\
4 & Attn K & 108.07 & 49.87 & 52.84 & 125.93 & 116.10 & 6.56 & 108.47 & 50.24 & 116.48 \\
4 & Attn V & 196.25 & 30.69 & 72.89 & 211.59 & 209.45 & 6.56 & 196.62 & 31.30 & 209.80 \\
4 & Attn O & 166.50 & 35.85 & 64.68 & 183.96 & 179.94 & 6.56 & 167.14 & 36.26 & 180.54 \\
4 & Norm   & 8.27   & 7.85  & 1.40  & 6.22   & 8.73   & 11.31& 6.10   & 7.32  & 5.42   \\
\midrule
5 & MLP Up   & 145.63 & 52.05 & 69.95 & 168.30 & 160.38 & 6.54 & 145.63 & 52.47 & 160.44 \\
5 & MLP Down & 245.78 & 52.32 & 81.35 & 259.95 & 255.42 & 6.57 & 245.92 & 52.77 & 255.57 \\
5 & MLP All  & 285.68 & 73.80 & 107.29& 309.68 & 301.60 & 9.27 & 285.81 & 74.42 & 301.76 \\
5 & Norm     & 7.81   & 6.27  & 2.73  & 9.48   & 7.97   & 11.31& 10.12  & 9.71  & 11.60  \\
\midrule
6 & Attn Q & 104.73 & 48.12 & 61.59 & 121.54 & 111.79 & 6.58 & 104.81 & 48.53 & 111.83 \\
6 & Attn K & 127.40 & 48.75 & 63.64 & 145.94 & 137.34 & 6.58 & 127.44 & 49.16 & 137.38 \\
6 & Attn V & 247.10 & 36.29 & 77.28 & 258.34 & 255.48 & 6.58 & 247.23 & 36.83 & 255.60 \\
6 & Attn O & 217.16 & 44.17 & 73.04 & 229.59 & 225.50 & 6.58 & 217.34 & 44.79 & 225.69 \\
6 & Norm   & 9.59   & 7.45  & 1.44  & 7.04   & 8.84   & 11.31& 6.47   & 7.10  & 6.73   \\
\midrule
7 & MLP Up   & 128.53 & 54.64 & 73.37 & 156.92 & 147.30 & 6.57 & 128.75 & 55.06 & 147.45 \\
7 & MLP Down & 292.75 & 61.59 & 120.01& 315.85 & 310.40 & 6.56 & 292.94 & 61.97 & 310.53 \\
7 & MLP All  & 319.73 & 82.33 & 140.66& 352.68 & 343.57 & 9.29 & 319.98 & 82.90 & 343.76 \\
7 & Norm     & 6.96   & 5.02  & 7.41  & 9.34   & 7.72   & 11.31& 8.04   & 10.77 & 14.29  \\
\bottomrule
\end{tabular}%
}
\label{tab:encoder_displacement_combined2}
\end{table*}

\newpage

\begin{figure}[ht]
    \centering
    \begin{subfigure}[t]{0.45\linewidth}
        \centering
        \includegraphics[width=\linewidth]{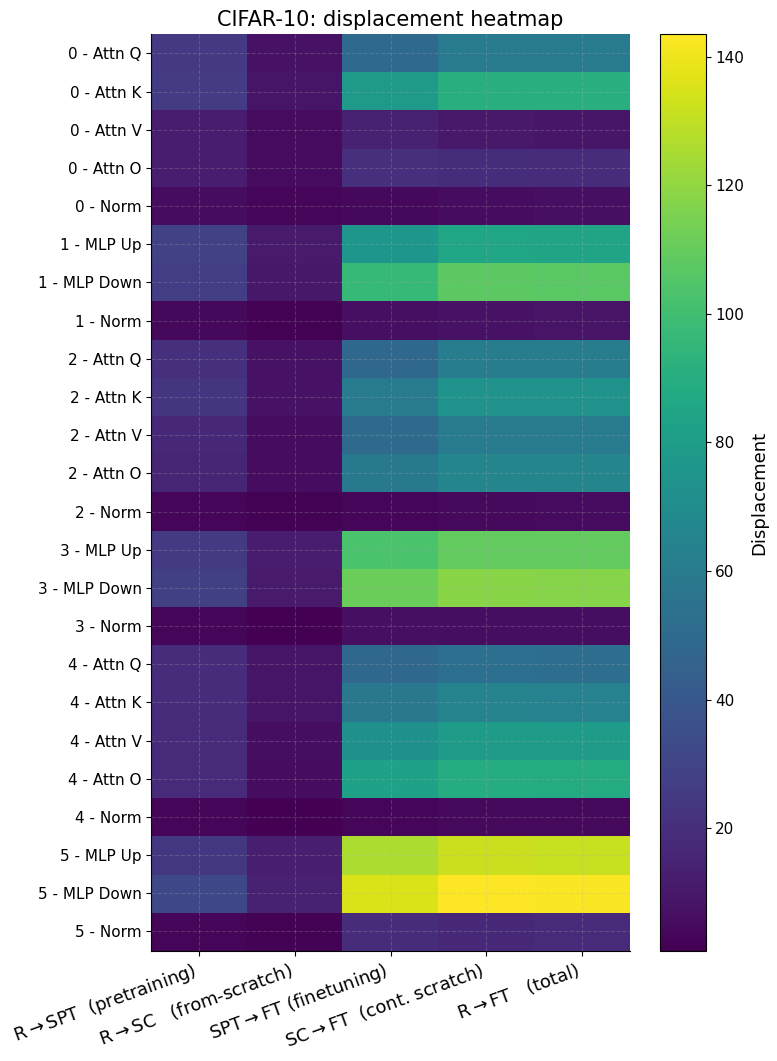}
        \caption{CIFAR10}
        \label{fig:heatmap_cifar}
    \end{subfigure}
    \hspace{0.02\linewidth}
    \begin{subfigure}[t]{0.45\linewidth}
        \centering
        \includegraphics[width=\linewidth]{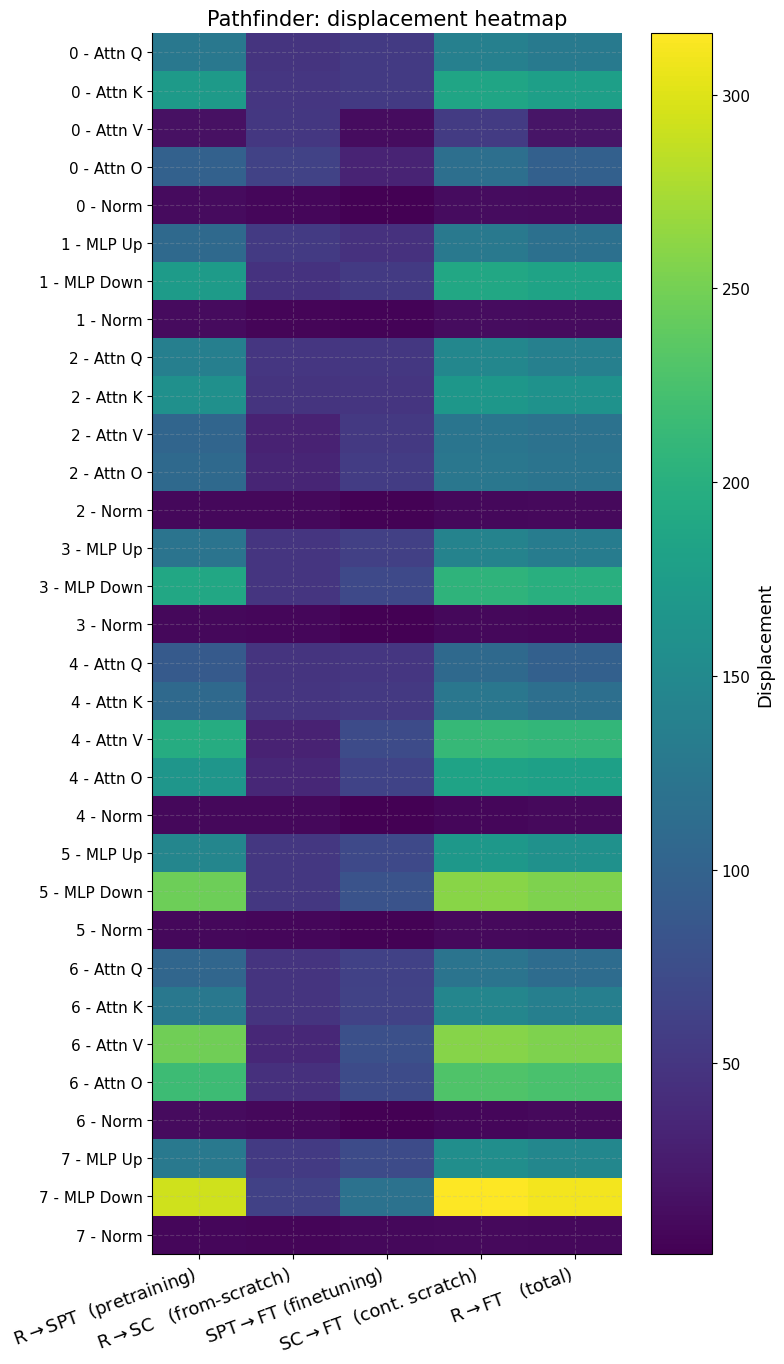}
        \caption{PathFinder}
        \label{fig:heatmap_path}
    \end{subfigure}
    \caption{\textit{Layer-wise parameter displacement across training trajectories.}
    R$\to$SC displacement is consistently smaller than R$\to$SPT and SPT$\to$FT,
    showing that supervised training from random initialization induces limited movement.
    MLP layers move substantially more than Attention projections, while Attention layers
    move little unless initialized from SPT.}
    \label{fig:displacement_heatmaps}
\end{figure}

\begin{figure}[ht]
    \centering
    \begin{subfigure}[t]{1\linewidth}
        \centering
        \includegraphics[width=\linewidth]{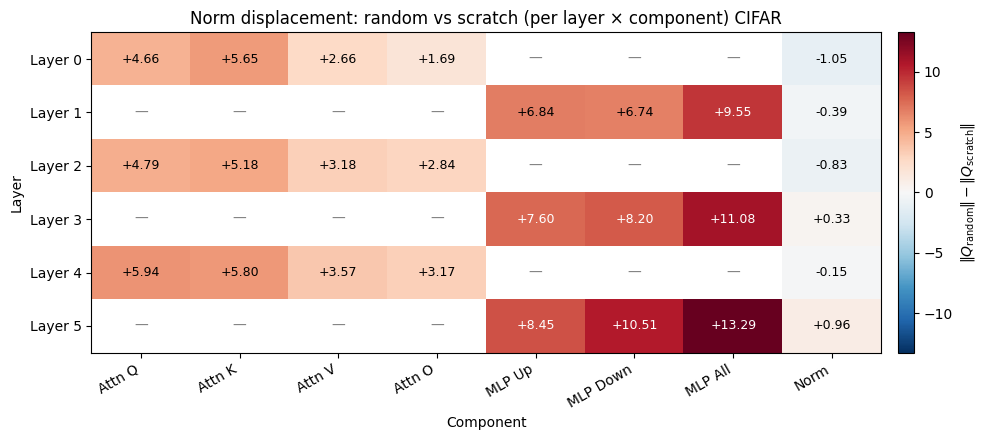}
        \caption{CIFAR10}
        \label{fig:delta_heatmap_cifar}
    \end{subfigure}
    \hfill
    \begin{subfigure}[t]{1\linewidth}
        \centering
        \includegraphics[width=\linewidth]{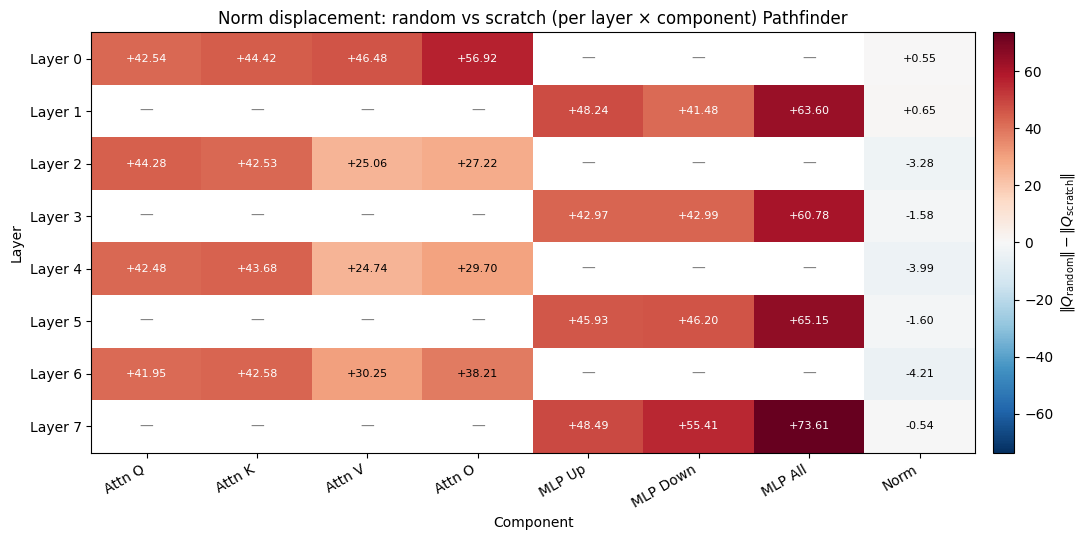}
        \caption{PathFinder}
        \label{fig:delta_heatmap_path}
    \end{subfigure}
    \caption{\textit{Delta between $\|Q_{\mathrm{random}}\|$ and $\|Q_{\mathrm{scratch}}\|$ across layers and components.}
    Norm displacement is largest for MLP All blocks and grows with depth, while normalization layers stay near zero (and even turn slightly negative in deeper layers).
    Attention projections move noticeably less than MLP blocks, consistent with the smaller-table observations.}
    \label{fig:delta_heatmaps}
\end{figure}
\clearpage
\newpage
\section{Synthetic Task}
\label{sec:synthetic_appendix}
We first describe the synthetic two-dimensional binary time-series classification dataset used in \S~\ref{sec:toy}, and then present additional plots supporting our claims.

\subsection{Task Description}
Each sample is a sequence $x = (x_j)_{j=0}^{L-1} \in \mathbb{R}^{L \times 2}$ with $L=100$ and a binary label $y \in \{0,1\}$. With
$$
a_0=(0.66,\,0.55), \quad
b_0=(-0.495,\,-0.605), \quad
a_1=(0.605,\,-0.55), \quad
b_1=(-0.55,\,0.605).
$$

\textbf{Conditioned on the binary label}, we define anchor points
$$
(c_1,c_2)=
\begin{cases}
(a_0,b_0), & y=0,\\
(a_1,b_1), & y=1.
\end{cases}
$$

For each sequence, we sample independently
$$
s \sim \mathcal{U}(0.12,0.28), \qquad
\phi \sim \mathcal{U}(0,2\pi), \qquad
\alpha \sim \mathcal{N}(1,0.1^2).
$$
Define normalized time $t_j = j/(L-1)$ and a ``warped'' time variable
$$
\tilde t_j = t_j
+ s \sin(2\pi t_j)\bigl(0.5 + 0.5 \cos(2\pi t_j)\bigr).
$$
The mixing coefficient and mean trajectory are defined as
$$
m_j = 0.5 + 0.35 \sin(2\pi \tilde t_j + \phi), \qquad \mu_j = m_j c_1 + (1-m_j) c_2.
$$
Finally, a smooth oscillatory drift is added: $d_j = 0.30 \bigl(\sin(2\pi \tilde t_j),\, \cos(2\pi \tilde t_j)\bigr)$. Observations are generated as
\[
x_j = \alpha (\mu_j + d_j) + \varepsilon_j + \delta_j,
\qquad
\varepsilon_j \sim \mathcal{N}(0,\sigma^2 I_2), \ \sigma=0.55.
\]

To introduce outliers, sample
$N_{\mathrm{sp}} \sim \mathrm{Unif}\{0,1,2,3\}$ and choose a random index set
$I \subset \{0,\dots,L-1\}$ with $|I|=N_{\mathrm{sp}}$.
For $j \in I$, draw $\delta_j \sim \mathcal{N}(0,I_2)$; otherwise set $\delta_j=0$.

\begin{figure}[h]
    \centering
    \includegraphics[width=0.5\linewidth]{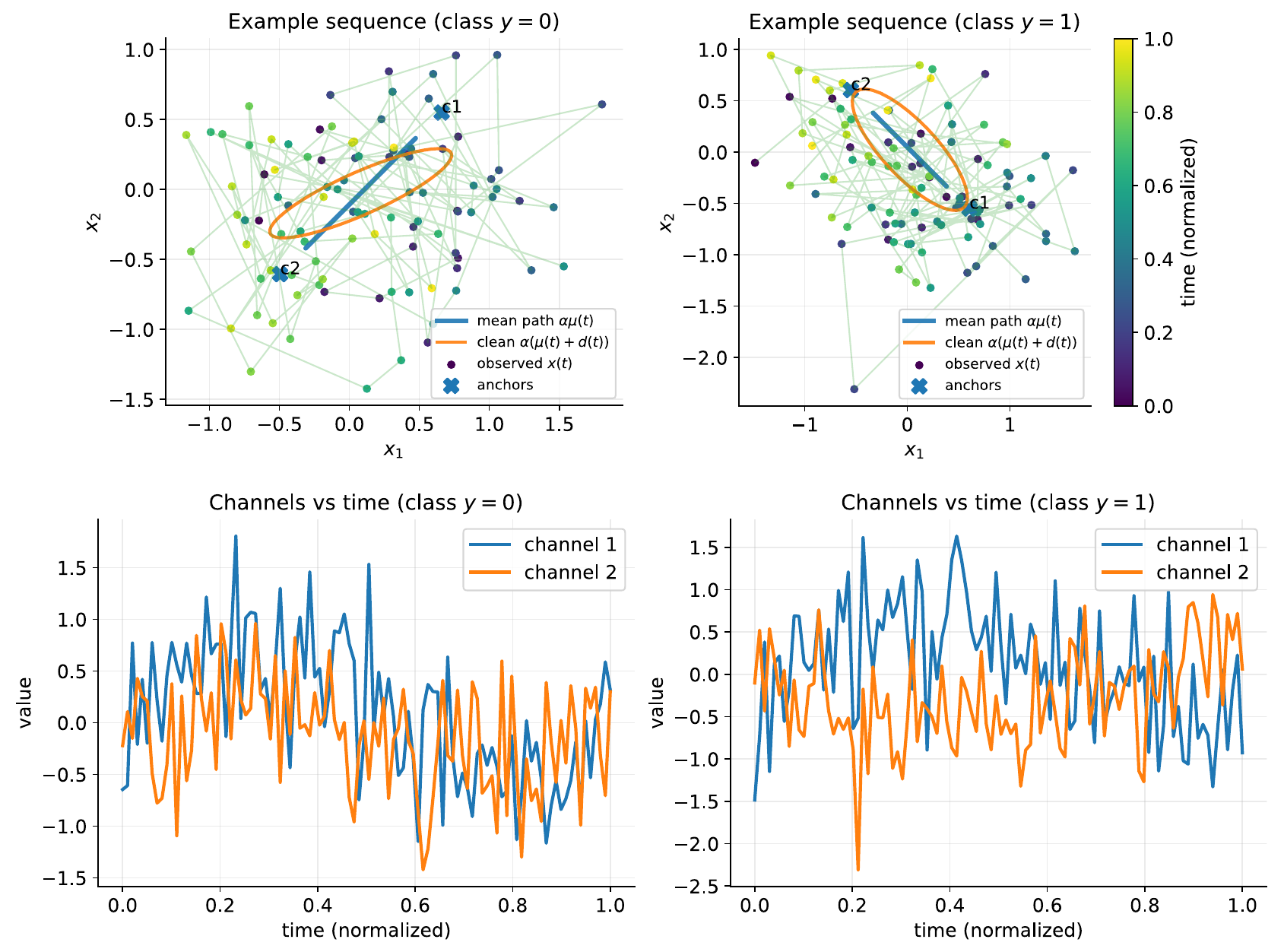}
    \caption{Illustration of the toy task for a random seed.}
    \label{fig:toy_task_def}
\end{figure}

\clearpage
\newpage
\subsection{Additional Plots for \S~\ref{sec:toy}}
\begin{figure}[h]
    \centering
    \includegraphics[width=0.8\linewidth]{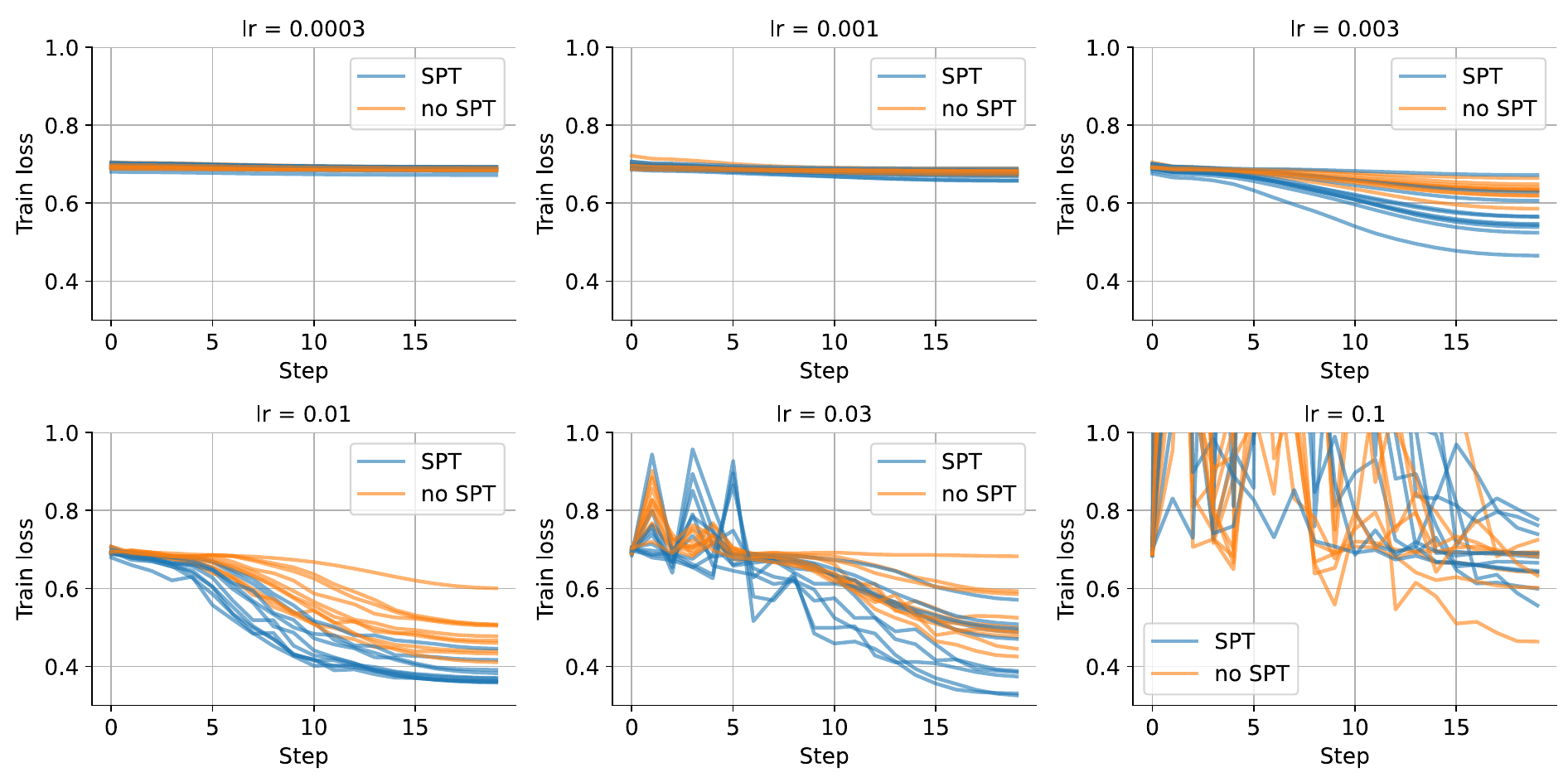}
    \caption{Complementing Fig.~\ref{fig:toy_main_tuning} with loss evolution over time. SPT initialization leads to faster and more stable convergence across seeds, particularly at intermediate learning rates, while differences diminish at very low learning rates and training becomes unstable at high learning rates.}
    \label{fig:toy_seeds}
\end{figure}

\begin{figure}[h]
    \centering
    \includegraphics[width=0.7\linewidth]{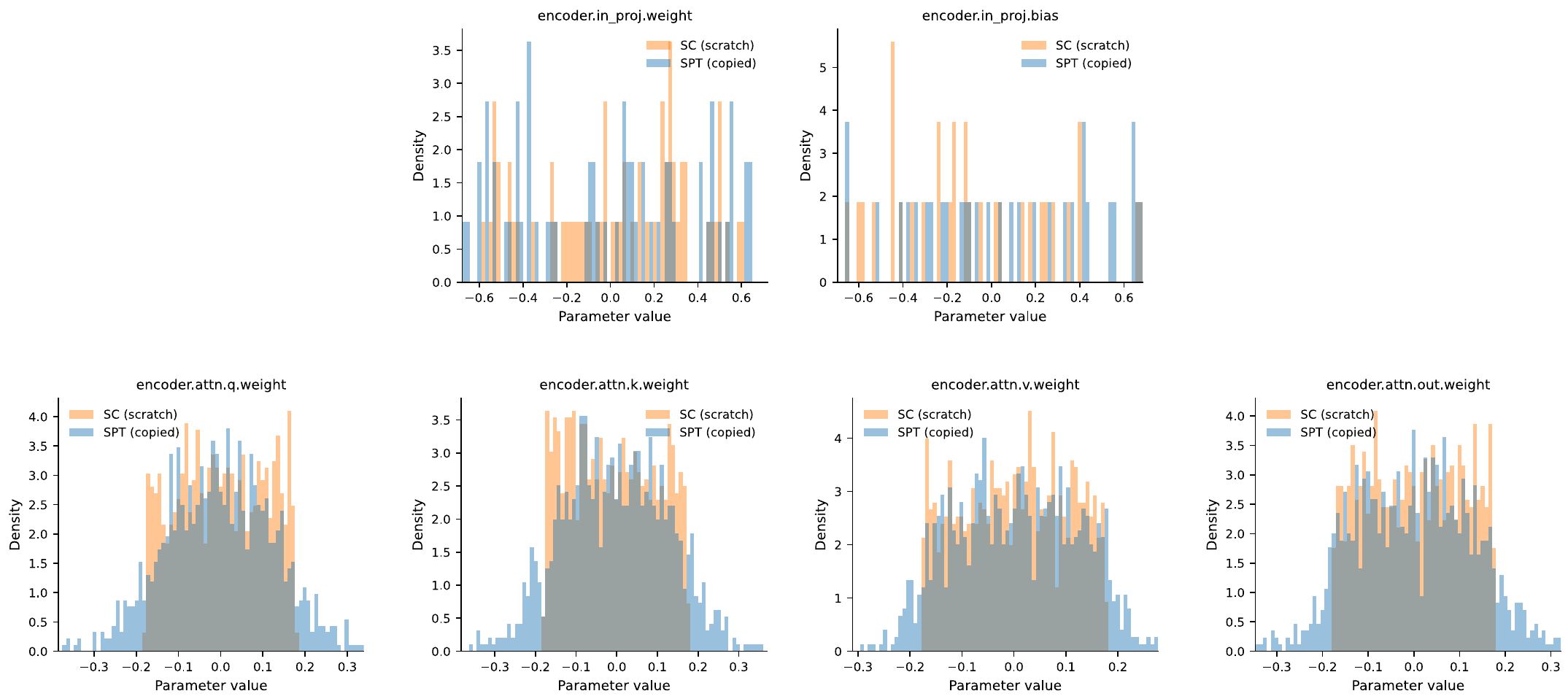}
    \caption{
    Weight distributions after SPT remain close to random initialization, except for mild broadening toward a Gaussian-like shape. Thus, the main structure is not visible in the marginal distributions of $W_Q$ or $W_K$, but in their product and its interaction with positional encodings.
    }
    \label{fig:toy_distributions}
\end{figure}

\begin{figure}[h]
    \centering
    \includegraphics[width=0.5\linewidth]{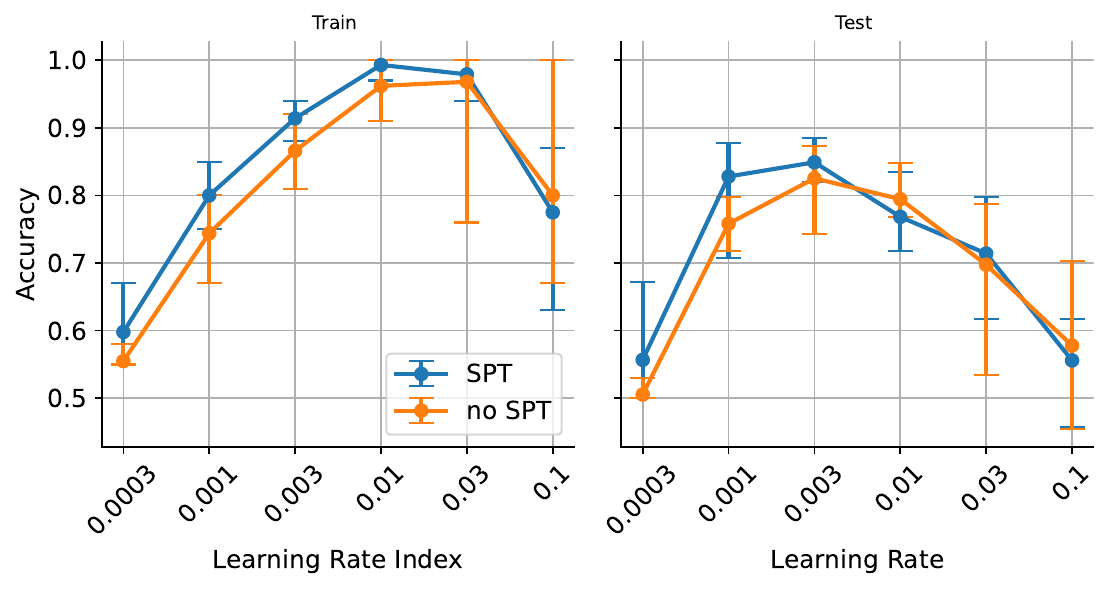}
    \caption{Same setting as Fig.~\ref{fig:toy_main_tuning}, but finetuning is performed for 80 epochs instead of 20. SPT initialization still achieves higher accuracy across most learning rates, although the gap narrows relative to Fig.~\ref{fig:toy_main_tuning}.}
    \label{fig:toy_longer}
\end{figure}

\clearpage
\newpage

\section{Proofs}
\label{sec:proofs}
Here we consider the attention mechanism 
$$
a_{ij}=\frac{\exp(s_{ij})}{\sum_{k=1}^L\exp(s_{ik})},
$$
where $s_{ij}$ are scores that define the attention pattern. For simplicity, we consider them independent of the input. Further, we consider learning a target pattern $s^*_{ij}=\Delta_{ij}$.

We first show in the next result that we can, without loss of generality, assume that $\Delta$ is centered.

\begin{lemma}[Row-centering is without loss of generality]
\label{lemma:rowcenter}
Let $\Delta\in\mathbb R^{L\times L}$ be any score direction, and define its row-centered version
$$
\bar\Delta_{ij}:=\Delta_{ij}-\frac1L\sum_{k=1}^L \Delta_{ik}, \qquad \text{which implies}\qquad \bar\Delta \mathbf 1=0.
$$
For any base scores $b_{ij}$ and any $\alpha\in\mathbb R$, the softmax attention weights generated by
$$
s^\Delta_{ij}(\alpha)=b_{ij}+\alpha \Delta_{ij}
$$
are identical to those generated by
$$
s^{\bar\Delta}_{ij}(\alpha)=b_{ij}+\alpha \bar\Delta_{ij}.
$$
Consequently, any forward pass and any loss depending on the scores only through softmax attention are identical for $\Delta$ and $\bar\Delta$. 
\end{lemma}

\begin{proof}
First, we quickly verify that centering implies $\bar\Delta\mathbf 1=0$
$$
\sum_{j=1}^L \bar\Delta_{ij}
=
\sum_{j=1}^L \Delta_{ij}
-
\sum_{j=1}^L \frac1L\sum_{k=1}^L \Delta_{ik}
=
\sum_{j=1}^L \Delta_{ij}
-
\sum_{k=1}^L \Delta_{ik}
=
0.
$$
Next, let $c_i:=\frac1L\sum_{k=1}^L \Delta_{ik}$. Then $\bar\Delta_{ij}=\Delta_{ij}-c_i$, so
$$s^{\bar{\Delta}}_{ij}(\alpha) = b_{ij} + \alpha\bar{\Delta}_{ij} = b_{ij} + \alpha(\Delta_{ij} - c_i) = s^\Delta_{ij}(\alpha) - \alpha c_i$$
For each fixed row $i$, the quantity $\alpha c_i$ is independent of $j$. Therefore
$$
\frac{\exp(s^{\bar\Delta}_{ij}(\alpha))}
{\sum_{k=1}^L \exp(s^{\bar\Delta}_{ik}(\alpha))}
=
\frac{\exp(s^\Delta_{ij}(\alpha)-\alpha c_i)}
{\sum_{k=1}^L \exp(s^\Delta_{ik}(\alpha)-\alpha c_i)}
=
\frac{e^{-\alpha c_i}\exp(s^\Delta_{ij}(\alpha))}
{e^{-\alpha c_i}\sum_{k=1}^L \exp(s^\Delta_{ik}(\alpha))}
=
\frac{\exp(s^\Delta_{ij}(\alpha))}
{\sum_{k=1}^L \exp(s^\Delta_{ik}(\alpha))}.
$$
Thus the attention weights are identical for all $\alpha$, and so are the resulting outputs and losses.
\end{proof}

In the following proposition, we consider optimization of the attention matrix towards a normalized latent score $\Delta$ through a single parameter $\alpha$.

\paragraph{Restated Proposition~\ref{prop:theory1}} Let $X_1,\ldots,X_L\in\mathbb R^d$ be centered random tokens with finite second moments, and define the autocorrelation matrix $C\in\mathbb R^{L\times L}$ as $C_{ij}:=\mathbb E[X_i^\top X_j]$. For a score pattern $\Delta\in\mathbb R^{L\times L}$, consider training a single parameter $\alpha$ controlling the attention score:
$$
s_{ij}(\alpha)=\alpha\Delta_{ij},\qquad
a_{ij}(\alpha)=\frac{\exp(s_{ij}(\alpha))}{\sum_{k=1}^L\exp(s_{ik}(\alpha))},
$$
with outputs $o_i(\alpha)=\sum_j a_{ij}(\alpha)X_j$ and mean-pooled representation $h(\alpha)=L^{-1}\sum_i o_i(\alpha)$. 

Consider the following losses and assume they are regular enough so that one can interchange expectation and differential computation:
$$
\mathcal L_{\rm sup}(\alpha)=\mathbb E[\ell(h(\alpha),Y)]
\qquad\text{and}\qquad
\mathcal L_{\rm SPT}(\alpha)=\frac1{2L}\sum_i\mathbb E[\|X_i-o_i(\alpha)\|^2].
$$
Let $\mathbf 1\in\mathbb R^L$ be the all-ones vector and define $\mathcal B:=\{\Gamma\in\mathbb R^{L\times L}: \Gamma^\top\mathbf 1=0\}$. Further, define $\langle A,B\rangle_F=\sum_{i,j}A_{ij}B_{ij}$ and $\bar \Delta = \Delta-\frac{1}{L}\Delta\mathbf 1\mathbf 1^\top$. For every $\Delta\in\mathcal B$, we have:
$$
\mathcal L_{\rm sup}'(0)=0,
\qquad
\mathcal L_{\rm SPT}'(0)=-\frac1{L^2}\langle \bar \Delta,C\rangle_F 
$$

\begin{proof}
Fix $\Delta\in\mathcal B$. Thanks to Lemma~\ref{lemma:rowcenter}, without loss of generality, we can assume $\Delta\in \tilde{\mathcal B}$
$$
\tilde{\mathcal B}:=\{\Gamma\in\mathbb R^{L\times L}:\Gamma\mathbf 1=0,\ \Gamma^\top\mathbf 1=0\}.
$$
First compute the derivatives of the attention weights. For each fixed row $i$,
$$
a_{ij}(\alpha)=\frac{e^{\alpha\Delta_{ij}}}{\sum_k e^{\alpha\Delta_{ik}}}.
$$
Therefore
$$
a_{ij}'(\alpha)
=
a_{ij}(\alpha)
\left(
\Delta_{ij}-\sum_k a_{ik}(\alpha)\Delta_{ik}
\right).
$$
At $\alpha=0$, $a_{ij}(0)=1/L$, so
$$
a_{ij}'(0)
=
\frac1L
\left(
\Delta_{ij}-\frac1L\sum_k\Delta_{ik}
\right).
$$
Since $\Delta\mathbf 1=0$, this simplifies to
$$
a_{ij}'(0)=\frac{\Delta_{ij}}{L}.
$$

Thus the derivative of the attention output is
$$
o_i'(0)
=
\sum_j a_{ij}'(0)X_j
=
\frac1L\sum_j\Delta_{ij}X_j.
$$
The derivative of the mean-pooled representation is then
$$
h'(0)
=
\frac1L\sum_i o_i'(0)
=
\frac1{L^2}\sum_i\sum_j\Delta_{ij}X_j
=
\frac1{L^2}\sum_j\left(\sum_i\Delta_{ij}\right)X_j.
$$
Since $\Delta^\top\mathbf 1=0$, we have $h'(0)=0$. Hence
$$
\mathcal L_{\rm sup}'(0)
=
\mathbb E[\nabla_h\ell(h(0),Y)^\top h'(0)]
=
0.
$$

Now consider SPT. At $\alpha=0$, all rows attend uniformly, so $o_i(0)=\bar X:=\frac1L\sum_k X_k$. Differentiating,
$$
\mathcal L_{\rm SPT}'(0)
=
\frac1L\sum_i
\mathbb E[(o_i(0)-X_i)^\top o_i'(0)].
$$
Substituting $o_i(0)=\bar X$ and $o_i'(0)=L^{-1}\sum_j\Delta_{ij}X_j$,
$$
\mathcal L_{\rm SPT}'(0)
=
\frac1{L^2}\sum_{i,j}\Delta_{ij}
\mathbb E[(\bar X-X_i)^\top X_j].
$$
Let $S_j:=\sum_k C_{kj}$. Since
$$
\mathbb E[\bar X^\top X_j]
=
\frac1L\sum_k C_{kj}
=
\frac{S_j}{L},
$$
we get
$$
\mathbb E[(\bar X-X_i)^\top X_j]
=
\frac{S_j}{L}-C_{ij}.
$$
Therefore
$$
\mathcal L_{\rm SPT}'(0)
=
\frac1{L^2}\sum_{i,j}\Delta_{ij}
\left(
\frac{S_j}{L}-C_{ij}
\right).
$$
The term involving $S_j$ vanishes because $\sum_i\Delta_{ij}=0$, and the result follows.
\end{proof}

\paragraph{Diagonal-masked variant.} It is possible to show a very similar result for the case where there are token-to-same-token interactions: $a_{ii}(\alpha)=0$ and, for $j\ne i$,
$$
a_{ij}(\alpha)=\frac{e^{\alpha\Delta_{ij}}}{\sum_{k\ne i}e^{\alpha\Delta_{ik}}}.
$$
Let us directly assume we work with doubly-normalized targets: $\Delta_{ii}=0$, $\sum_{j\ne i}\Delta_{ij}=0$, and $\sum_{i\ne j}\Delta_{ij}=0$. Then:
$$
a'_{ij}(0)
=
\frac1{L-1}
\left(
\Delta_{ij}-\frac1{L-1}\sum_{k\ne i}\Delta_{ik}
\right)
=
\frac{\Delta_{ij}}{L-1}
\qquad (j\ne i).
$$
Hence $o_i'(0)=(L-1)^{-1}\sum_{j\ne i}\Delta_{ij}X_j$, and so
$$
h'(0)
=
\frac1{L(L-1)}
\sum_j
\left(\sum_{i\ne j}\Delta_{ij}\right)X_j
=
0.
$$
Thus again we get $\mathcal L_{\rm sup}'(0)=0$. For SPT, $o_i(0)=\bar X_{-i}:=(L-1)^{-1}\sum_{k\ne i}X_k$, so
$$
\mathcal L_{\rm SPT}'(0)
=
\frac1{L(L-1)}
\sum_{i,j\ne i}
\Delta_{ij}
\mathbb E[(\bar X_{-i}-X_i)^\top X_j].
$$
For $j\ne i$, writing $S_j:=\sum_k C_{kj}$,
$$
\mathbb E[(\bar X_{-i}-X_i)^\top X_j]
=
\frac{S_j-C_{ij}}{L-1}-C_{ij}
=
\frac{S_j}{L-1}-\frac{L}{L-1}C_{ij}.
$$
Therefore~(the first term vanishes by off-diagonal column-centering)
$$
\mathcal L_{\rm SPT}'(0)
=
\frac1{L(L-1)^2}
\sum_j S_j\sum_{i\ne j}\Delta_{ij}
-
\frac1{(L-1)^2}
\sum_{i\ne j}\Delta_{ij}C_{ij}
=
-\frac1{(L-1)^2}
\sum_{i\ne j}\Delta_{ij}C_{ij}.
$$

\begin{figure}
    \centering
    \includegraphics[width=0.99\linewidth]{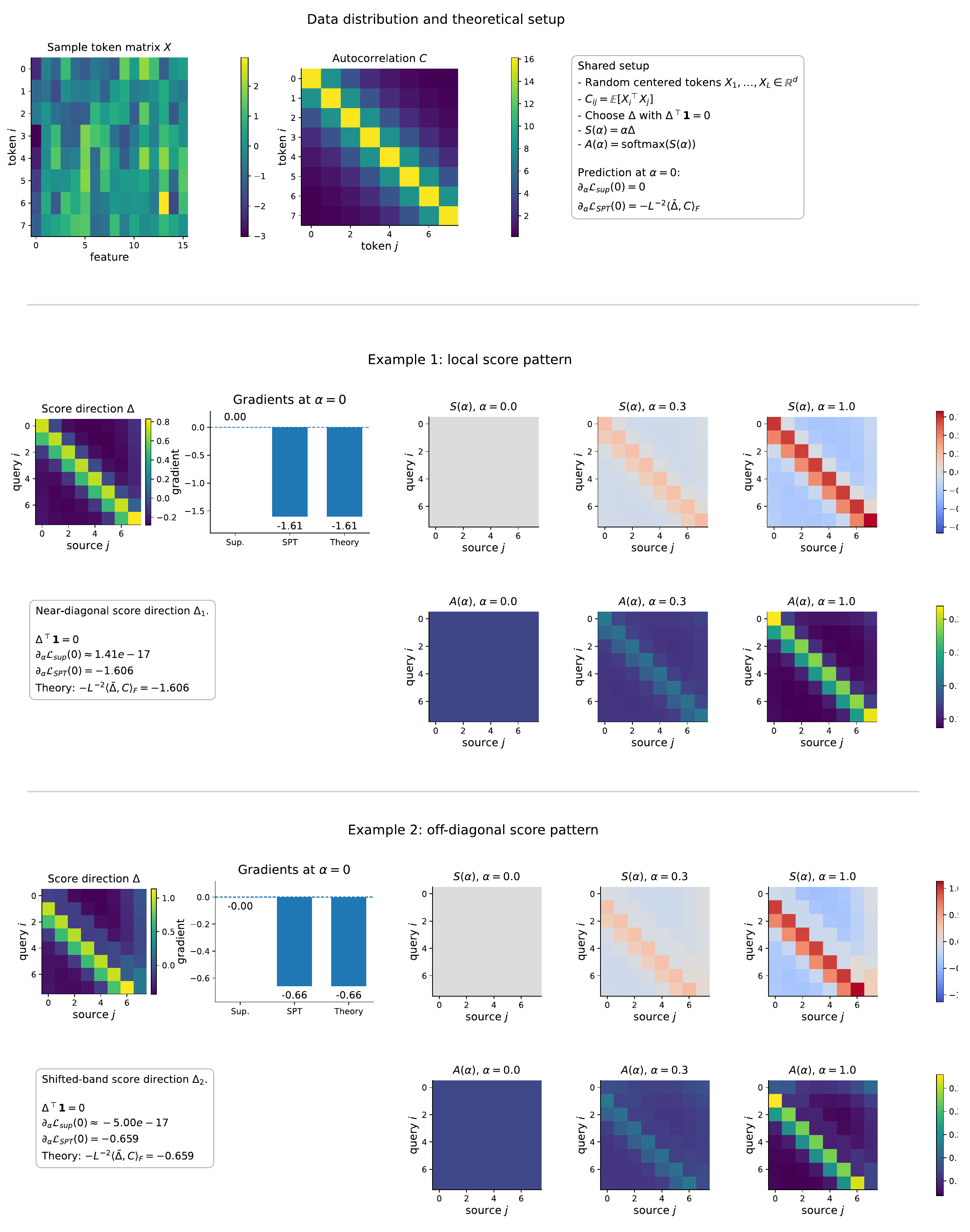}
    \caption{Verification of Proposition~\ref{prop:theory1} on 2 example patterns under randomly sampled tokens.}
    \label{fig:spt_theory}
\end{figure}



\end{document}